\documentclass[conference]{IEEEtran}
\IEEEoverridecommandlockouts

\usepackage{cite}
\ifCLASSINFOpdf

\else
\fi
\usepackage{color}
\usepackage[framemethod=tikz]{mdframed}

\usepackage{amssymb}
\usepackage[cmex10]{amsmath}
\usepackage{amsfonts}
\usepackage[numbers,sort&compress]{natbib}
\usepackage{graphicx}
\usepackage{subfigure}
\usepackage{float}
\usepackage{url}
\usepackage{algorithm}
\usepackage{algpseudocode}
\usepackage{multirow}
\usepackage{flushend} 
\usepackage{mathrsfs}
\usepackage{amsthm}
\usepackage{bm}
\usepackage{booktabs}
\usepackage{threeparttable}
\theoremstyle{plain}
\newtheorem{assumption}{Assumption}
\newtheorem{definition}{Definition}
\newtheorem{lemma}{Lemma}

\newtheorem{theorem}{Theorem}
\newtheorem{corollary}{Corollary}
\newtheorem{remark}{Remark}
\newenvironment{thmbis}[1]
{\renewcommand{\thetheorem}{\ref{#1}$'$}%
\addtocounter{theorem}{-1}%
\begin{theorem}}
{\end{theorem}}

\def\BibTeX{{\rm B\kern-.05em{\sc i\kern-.025em b}\kern-.08em
    T\kern-.1667em\lower.7ex\hbox{E}\kern-.125emX}}

\begin{document}
\title{Depersonalized Federated Learning: Tackling Statistical Heterogeneity by Alternating Stochastic Gradient Descent}
\author{\IEEEauthorblockN{Yujie Zhou$^{1, 2, 3}$, Zhidu Li$^{1, 2, 3}$, Tong Tang$^{1, 2, 3}$, Ruyan Wang$^{1, 2, 3}$}\\
\IEEEauthorblockA{
$^{1}$ Chongqing University of Posts and Telecommunications, School of Communication and Information Engineering, China\\
$^{2}$ Advanced Network and Intelligent Interconnection Technology Key Laboratory of Chongqing Education Commission of China \\
$^{3}$ Key Laboratory of Ubiquitous Sensing and Networking in Chongqing, China\\
Email: lizd@cqupt.edu.cn}
\thanks{This work was supported in part by the National Natural Science Foundation of China under grants 61901078, 61871062,61771082 and U20A20157, and in part by Natural Science Foundation of Chongqing under grant cstc2020jcyj-zdxmX0024, and in part by  University Innovation Research Group of Chongqing under grant CXQT20017.}
       }

\maketitle
\thispagestyle{empty}
\begin{abstract}
    Federated learning (FL), which has gained increasing attention recently, enables distributed devices to train a common machine learning (ML) model for intelligent inference cooperatively without data sharing.
    However, problems in practical networks, such as non-independent-and-identically-distributed (non-iid) raw data and limited bandwidth, give rise to slow and unstable convergence of the FL training process.
    To address these issues, we propose a new FL method that can significantly mitigate statistical heterogeneity through the \emph{depersonalization mechanism.}
    Particularly, we decouple the global and local optimization objectives by alternating stochastic gradient descent, thus reducing the accumulated variance in local update phases to accelerate the FL convergence.
    Then we analyze the proposed method in detail to show the proposed method converging at a sublinear speed in the general non-convex setting.
    Finally, numerical results are conducted with experiments on public datasets to verify the effectiveness of our proposed method.
\end{abstract}
\begin{IEEEkeywords}
    Federated learning, depersonalization mechanism, statistical heterogeneity, convergence analysis
\end{IEEEkeywords}

\section{Introduction}
Due to a tremendous amount of data in edge devices, machine learning (ML) as a data-driven technology is generally used to enhance the intelligence of applications and networks \cite{wang2019edge, zhou2019edge}.
However, traditional ML requiring centralized training is unsuitable for the scenario because of privacy concerns and communication costs in raw data transmission.
Thus, as a distributed optimization paradigm, federated learning (FL), is designed to train ML models across multiple clients while keeping data decentralized.

To train ML models distributively, we can directly use the classical Parallel-SGD \cite{dekel2012optimal}, i.e., each client calculates the local stochastic gradient to the central server for getting the aggregated gradient at each iteration. Nevertheless, performing the procedure still leads to unaffordable communication costs, especially in the case of training large primary models such as deep neural networks.
Then to reduce the costs, a popular algorithm FedAvg \cite{FedAvg} was proposed, which means training individual models via several local SGD steps and uploading them in place of gradients to the central server in aggregation.
Despite FedAvg successfully reducing the communication overhead several times of Parallel-SGD, some key challenges emerge in deploying the framework:
(i) As massive clients may join in an FL training process, it is impractical for communication links to support all nodes to upload data simultaneously. 
(ii) As participators come from various regions, data on all clients are usually non-independent-and-identically-distributed (non-iid, known as statistical heterogeneity).
Recently, some efforts have been devoted to analyzing and improving FL (with (i) partial communication, a.k.a. client scheduling) performance on (ii) non-iid data.
Works \cite{koloskova2020unified, FAC, AFL} studied on FedAvg convergence.
Then \cite{nguyen2020fast, FedProx, SCAFFOLD, TACS} proposed FedAvg-based methods for incremental performance enhancement by update-rule or sampling policy modifications.
For instance, in \cite{FedProx}, the proposed FedProx introduced a proximal operator to obtain surrogate local objectives to tackle the heterogeneity problem empirically.
Then unlike the above works that focus on global performance improvement, other studies \cite{li2021ditto, t2020personalized, 9766407} tended to generate a group of personalized FL models in place of a single global model for all clients on non-iid data to ensure fairness and stylization.
For example, in \cite{li2021ditto}, the authors proposed a common personalized FL framework with inherent fairness and robustness, and \cite{t2020personalized} raised a bi-level learning framework for extracting personalized models from the global model.

\emph{Note that extra local information is implicit in customized FL models generated by personalized FL approaches.} While utilizing this information may be beneficial to reduce the negative impact caused by (i) client sampling and (ii) statistical heterogeneity.
Thus in this paper, we are inspired to devise a new method to improve global FL performance that modifies the local-update-rule by reversely using model-customization techniques \cite{li2021ditto, t2020personalized, 9766407}.
To take advantage of this personalization information, we design a double-gradient-descent rule in the local update stage that each client generates two decoupled local models (rather than an original one) to separate the global update direction from the local one.
In particular, the personalized local model is obtained by directly optimizing the local objective, while the globalized local model is obtained by \emph{subtracting} the personalized local model from the original one.
Therefore, each sampled client can upload \emph{its globalized model in place of the original local one} to reduce the accumulated local deviations for convergence acceleration and stabilization. We summarize key contributions as follows:
\begin{itemize}
    \item We propose a novel method called FedDeper to improve the FL performance on non-iid data by the depersonalization update mechanism, which can be widely adapted to a variety of scenarios.
    \item We theoretically analyze the convergence performance of our proposed method for the personalized and aggregated models in the general non-convex setting.
    \item We provide relevant experimental results to evaluate the convergence performance of our proposed algorithm versus baselines and study the impact factors of convergence.
\end{itemize}

The remainder of this paper is organized as follows.
We start by discussing the impact of data heterogeneity on the canonical FedAvg method in Section II.
Then, we propose a new FedDeper method in Section III and analyze its convergence in Section IV.
Next, we present and discuss experimental results in Section V.
Finally, we conclude the paper in Section VI.

\section{Preliminaries and Backgrounds}
In an FL framework, for all participating clients (denoted by $\mathcal{N}$ with the cardinal number $n:= |\mathcal{N}|$), we have the following optimization objective:
\begin{equation}\label{1.global_obj}
    \min_{\bm{x}\in \mathbb{R}^d} f(\bm x) := \frac{1}{n}\sum\nolimits_{i \in \mathcal{N}} f_i(\bm x)
\end{equation}
where $d$ denotes the dimension of the vector $\bm x$, and $f_i(\bm x) := {\mathbb{E}}_{\bm {\vartheta}_i \sim D_i}[f(\bm x; \bm {\vartheta}_i)]$ represents the local objective function on each client $i$.
Besides, $f_i$ is generally the loss function defined by the local ML model, and $\bm \vartheta_i$ denotes a data sample belonging to the local dataset $D_i$.
In this paper, we mainly deal with Problem (\ref{1.global_obj}) \cite{nguyen2020fast, FedProx, SCAFFOLD}, and all the results can be extended to the weighted version by techniques in \cite{FAC, TACS}.
We depict a round of the typical algorithm FedAvg to solve (\ref{1.global_obj}) as three parts: In the $k$-th round, (i) {Broadcasting:}
The server uniformly samples a subset of $m$ clients (i.e.,  $\mathcal{U}^k \subseteq \mathcal{N}$ with $m := |\mathcal{U}^k| \leq n, \forall k \in \{0,1,...,K-1\}$ for any integer $K \geq 1$) and broadcasts the aggregated global model $\bm x^k$ to client $i \in \mathcal{U}^k$.
(ii) {Local Update:}
Each selected client $i$ initializes the local model $\bm {v}_{i,0}^k$ as $ \bm x^k$ and then trains the model by performing stochastic gradient descent (SGD) with a step size $\eta$ on $f_i(\cdot)$,
\begin{equation}\label{2.fedavg.localupdate}  
    \bm {v}_{i,j+1}^k \gets \bm {v}_{i,j}^k - \eta {g_i}(\bm {v}_{i,j}^k),\ \forall j \in \{0,1,...,\tau-1\},
\end{equation}
where $\bm {v}_{i,j}^k$ denotes the updated local model in the $j$-th step SGD and ${g_i}(\cdot)$ represents the stochastic gradient of $f_i(\cdot)$ w.r.t. $\bm {v}_i$.
While the number of local steps reaches a certain threshold $\tau$, client $i$ will upload its local model to the server.
(iii) {Global Aggregation:}
The server aggregates all received local models to derive a new global one for the next phase,
\begin{equation}\label{2.fedavg.aggregation}  
    \begin{aligned}
        &\bm {x}^{k+1} \gets \frac{1}{m} \sum\nolimits_{i \in \mathcal{U}^k}\bm{v}^{k}_{i,\tau}. \\
    \end{aligned}
\end{equation}  
We complete the whole process when the number of communication rounds reaches the upper limit $K$, and obtain a global model trained by all participating clients.

Note that the stochastic gradient ${g_i}(\cdot)$ in Process (\ref{2.fedavg.localupdate}) can be more precisely rewritten as ${g_i}(\cdot) = \nabla f(\cdot; \bm {\vartheta}_i)$ with $\bm {\vartheta}_i \sim D_i$.
Since the high heterogeneity, local datasets ${D_i}_{i\in \mathcal{N}}$ obey unbalanced data distributions, and the corresponding generated gradients are consequently different in expectation:
\begin{equation}\label{2.fedavg.expgraddev}  
    \begin{aligned}
        &{\mathbb{E}}_{\bm {\vartheta}_i \sim D_i}[f(\cdot; \bm {\vartheta}_i)] \neq {\mathbb{E}}_{\bm {\vartheta}_j \sim D_j}[f(\cdot; \bm {\vartheta}_j)], \ {\forall i,j \in \mathcal{N}, i \neq j}.\\
    \end{aligned}
\end{equation}  
That means performing SGD (\ref{2.fedavg.localupdate}) with (\ref{2.fedavg.expgraddev}) leads to each client tending to find its local solution $\bm v_i^*$ with $\nabla f_i(\bm v_i^*) = 0$ deviating from the global one $\bm x^*$ with $\nabla f(\bm x^*) = 0$, where always holds the optimization objective inconsistency $ \cap_{i\in\mathcal{N}}\ker \nabla f_i = \varnothing,\ \forall i \in \mathcal{N}$ hence resulting in slow convergence.
Moreover, in practical deployed FL frameworks, the number of participators is always much large while the bandwidth or communication capability of the server is limited, i.e., only a small fraction of clients can be selected to join a training round: $m \ll n$.
This fact (partial communication) aggravates the inconsistency of local models, thus further leading to unreliable training and poor performance.

\section{Federated Learning with Depersonalization}
\begin{figure}[t]
    \centering
    \includegraphics[width=1\columnwidth]{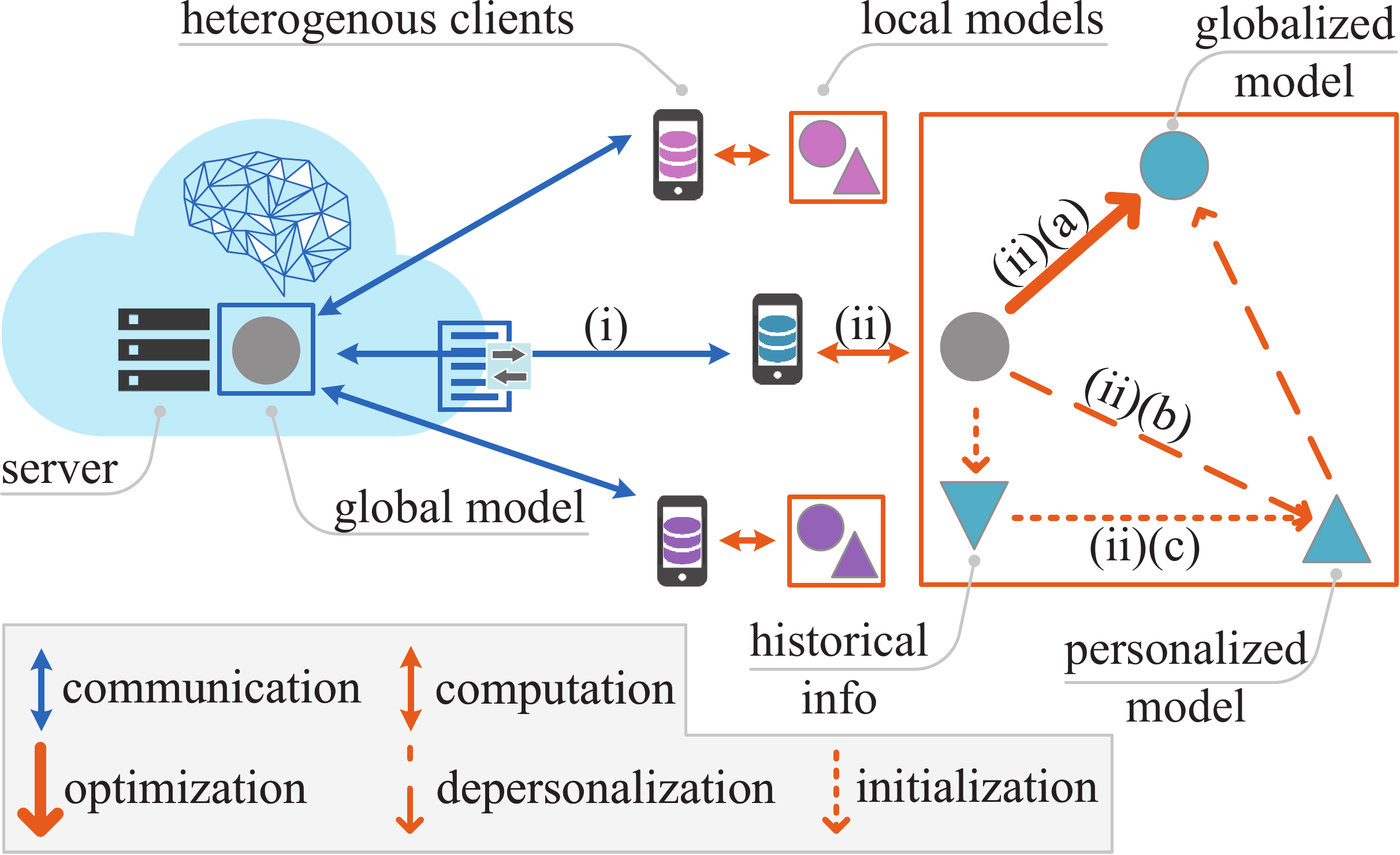}
    \caption{Federated learning with depersonalization: (i) communication (broadcasting \& aggregating), (ii) computation (local updating): (a) optimization, (b) depersonalization, and (c) initialization. Indeed, mechanism (a) integrates (b) which integrates (c), as (a) $\supset$ (b) $\supset$ (c). }\label{sys}
\end{figure}
To alleviate the negative impact of non-iid data and partial communication on FL, we propose a new Depersonalized FL (FedDeper) algorithm.
In brief, we aim to generate local approximations of the global model on clients, then upload and aggregate them in place of the original local models for stabilization and acceleration, as shown in Fig. \ref{sys}.

\subsection{Decoupling Global and Local Updating}
Recall that performing (\ref{2.fedavg.localupdate}) aims to minimize the local objective $f_i(\cdot)$ that usually disagrees with the global one (\ref{1.global_obj}) resulting in slow convergence.
To deal with this issue, we propose a new depersonalization mechanism to decouple the two objectives.
In particular, to better optimize the objective $f(\cdot)$, we induce a more \emph{globalized local model} in place of the original uploaded one to mitigate the local variance accumulation in aggregation rounds.
Different from Process (\ref{2.fedavg.localupdate}), we perform SGD on the surrogate loss function in each selected client $i$,
\begin{equation}\label{3.feddeper.localsurro}
    f_i^{\rho}(\bm y_i) := f_i(\bm y_i) + \frac{\rho}{2\eta}\|{\bm{v}_i}+\bm{y}_i-2\bm {x}\|^2,
\end{equation}  
where $\frac{\rho}{2\eta}$ is a constant for balancing the two terms, and ${\bm{v}_i}$ fixed in updating $\bm y_i$ denotes the  \emph{personalized local model} (the analogue of the original local model), which aims to reach the local optimum $\bm v_i^*$ via (\ref{2.fedavg.localupdate}).
Thus, we expect to obtain two models in the phase. 
The one $\bm v_i$ is kept locally for searching the local solution $\bm v_i^*$ while the other $\bm y_i$ estimating the global model locally (i.e., $\bm y_i^* \approx \bm x^*$) is uploaded to the aggregator to accelerate FL convergence.

\subsection{Using Local Information Reversely with Regularizer}
\begin{figure}[t]
    \centering
    \includegraphics[width=1\columnwidth]{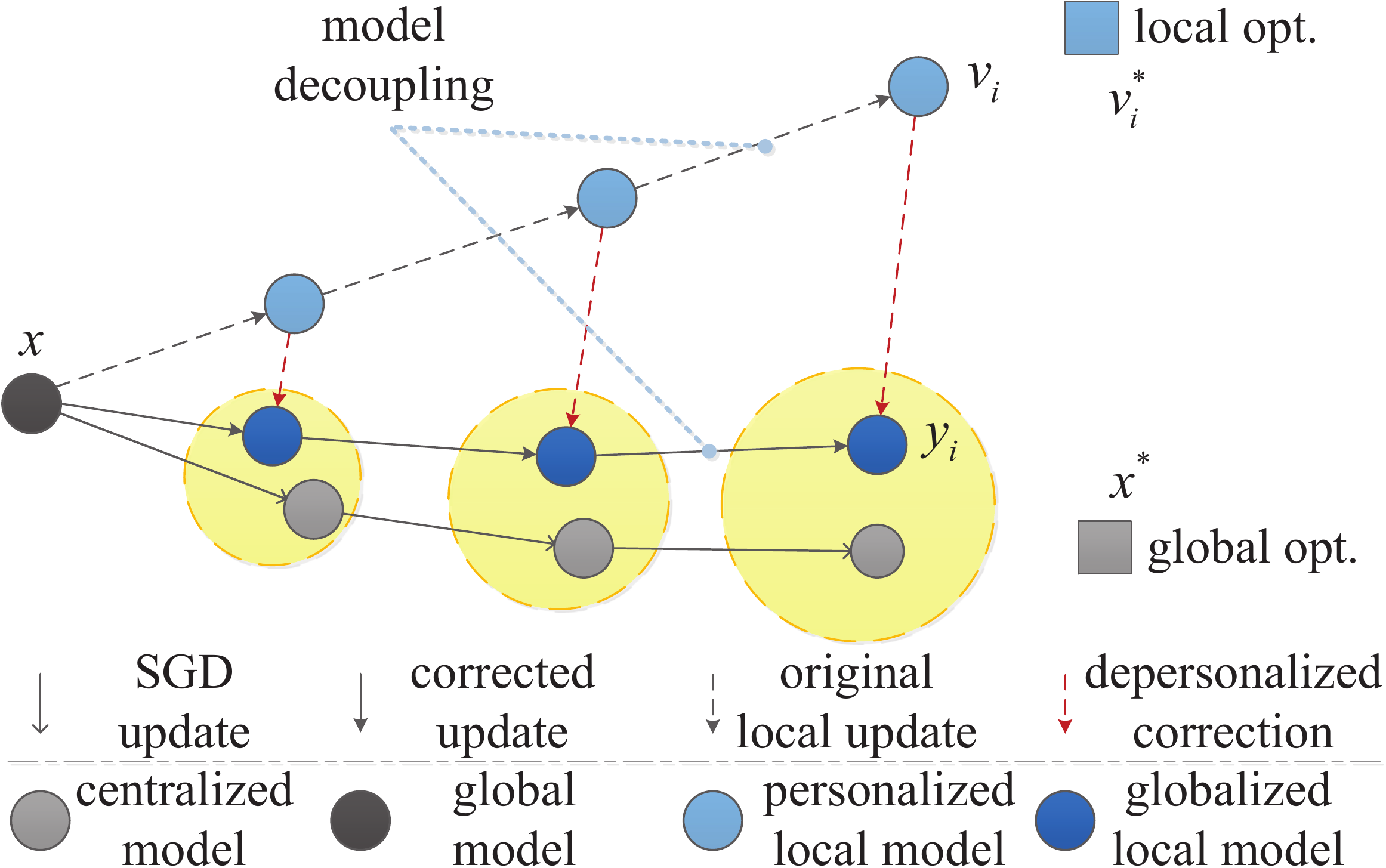}
    \caption{The local update phase of FedDeper: each selected client alternately updates globalized and personalized models in a round. The original local (or personalized) update aims to reach the local optimum $\bm v_i^*$ while the corrected update moves around the SGD update towards the global optimum $\bm x^*$ by reversely local update (depersonalization mechanism).}\label{localsys}
\end{figure}
As shown in Fig. \ref{localsys}, the globalized model $\bm y_i$ is updated by using the personalized one $\bm v_i$ reversely.
To minimize (\ref{3.feddeper.localsurro}), the value of $\bm y_i$ is restricted to a place slightly away from the local optimum with the regularizer $\|{\bm{v}_i}+\bm{y}_i-2\bm {x}\|^2$.
More specifically, we regard ${\bm v}_i - \bm x$ and $\bm y_i - \bm x$ as two directions in the update. 
Since ${\bm v}_i$ is a personalized solution for the client, we note that ${\bm v}_i - \bm x$ contains abundant information about local deviations.
To avoid introducing overmuch variance, we give a penalty to term $\bm y_i - \bm x$ that reflects $\bm y_i$ to the opposite direction of ${\bm v}_i - \bm x$.
Nevertheless, in suppressing bias with the regularizer $\|{\bm{v}}_i+\bm{y}_i-2\bm {x}\|^2$, we also eliminate the global update direction implied in ${\bm v}_i - \bm x$, which further interprets the necessity of carefully tuning on $\rho, \eta$ (trade off variance reduction and convergence acceleration).

\subsection{Retaining Historical Information for Personalized Model}
In the current local update stage, model $\bm y_i$ is initialized as the received global model $\bm x$ while $\bm {v}_i$ is initialized as the trained $\bm y_i$ in the previous stage.
And then they are updated alternately by first-order optimizers, i.e., each selected client $i$ performs SGD on $f_i(\cdot)$ as (\ref{2.fedavg.localupdate}) to obtain a personalized local model and on (\ref{3.feddeper.localsurro}) to obtain a locally approximated globalized model, respectively.
However, this initialization policy results in the new $\bm v_i$ forgetting all accumulated local information contained in the previous $\bm v_i$.
To make the best use of the historical models, we let $\bm {v}_i$ partially inherit the preceding value, i.e.,
\begin{equation}\label{3.feddeper.mix}
    {\bm{v}}_{i,0}^{k+1}\gets (1-\lambda){\bm{v}}_{i,\tau}^{k} + \lambda\bm{y}_{i,\tau}^k,
\end{equation}  
where ${\bm{v}}_{i,\tau}^{k}$, $\bm{y}_{i,\tau}^k$ are trained models in the $k$-th round, ${\bm{v}}_{i,0}^{k+1}$ is the initial model in the $k+1$-th round, and $\lambda \in [\frac{1}{2},1]$ is the mixing rate controlling the stock of local deviation information. 
To be specific, $\lambda$ limits the distance between the initial $\bm v_i$ and $\bm y_i$ within a certain range to avoid destructively large correction generated by $\|{\bm{v}_i}+\bm{y}_i-2\bm {x}\|^2$ since monotonically increasing difference between $\bm v_i$ and $\bm x$ (e.g., $\|\bm v_i -\bm x \|$) in updating.
\begin{remark}
    {\rm
    If $\lambda$ is set in the defined finite interval $[\frac{1}{2}, 1]$, we claim there exists suitable $\eta$, $\rho$ enable the global model $\bm x$ to converge to the global optimum.
    }
\end{remark}

\addtolength{\topmargin}{0.01in}
\subsection{Procedure of FedDeper and Further Discussion}
\begin{algorithm}[t]
    \caption{FedDeper: Depersonalized Federated Learning}\label{alg}
    {\bf Input:}
    learning rate $\eta$, penalty $\rho$, mixing rate $\lambda$, local step $\tau$, total round $K$, initialized models $\bm {x}^0=\bm {y}_{0,0}^0 =\bm {v}_{0,0}^0$
    \begin{algorithmic}[1]
    \For {each round $ k = 0,1,...,K-1 $}
        \State sample clients $\mathcal{U}^k \subseteq \mathcal{N}$ uniformly
        \State \textbf{send} $\bm {x}^k$ to selected clients $i \in \mathcal{U}^k$
        \For  {each client $i \in \mathcal{U}^k$ in parallel}
        \State initialize $ \bm {y}_{i,0}^k \gets \bm {x}^k$
        \For {$ j=0,1,...,\tau-1 $}
            \State $\bm {y}_{i,j+1}^k \gets \bm {y}_{i,j}^k - \eta {g_i^{\rho}}(\bm y_{i,j}^k)$\label{alg.local1}
            \State ${\bm{v}}_{i,j+1}^k \gets {\bm{v}}_{i,j}^k  - \eta {g_i}({\bm v}_{i,j}^k)$\label{alg.local2}
        \EndFor
        \State ${\bm{v}}_{i,0}^{k+1}\gets (1-\lambda){\bm{v}}_{i,\tau}^{k} + \lambda\bm{y}_{i,\tau}^k$ \label{alg.mvavg}
        \State \textbf{send} $\bm {y}_{i,\tau}^k - \bm {x}^k$ to server
        \EndFor
        \State each client $i \in  \complement_\mathcal{N}\mathcal{U}^k$ updates $\bm {v}_{i,0}^{k+1} \gets \bm {v}_{i,0}^{k}$ \label{alg.else}
        \State $\bm {x}^{k+1} \gets \bm{x}^{k} + \frac{1}{|\mathcal{U}^k|} \sum\nolimits_{i \in \mathcal{U}^k}(\bm{y}^{k}_{i,\tau} - \bm{x}^{k})$ \label{alg.agg}
    \EndFor
    \end{algorithmic}
\end{algorithm}
The proposed method is summarized as Algorithm \ref{alg}.
In Lines \ref{alg.local1}-\ref{alg.local2}, we update the globalized local model $\bm y$ and the personalized one ${\bm v}$ alternately.
Line \ref{alg.local1} shows the $j$-th step of local SGD where ${\bm v}$ is involved in the stochastic (mini-batch) gradient ${g_i^{\rho}}$ of $f_i^{\rho}$.
Line \ref{alg.local2} shows a step of SGD for approaching the optimum of the local objective. 
In Line \ref{alg.mvavg}, we initialize the personalized model ${\bm v}$ for the next round of local update with the mixing operation (\ref{3.feddeper.mix}).
In Line \ref{alg.else}, client $i \in \complement_\mathcal{N}\mathcal{U}^k$ skips the current round and only updates superscripts of variables.
In Line \ref{alg.agg}, the server receives and aggregates globalized local models from selected clients.

The proposal of the regularizer is inspired by FedProx. 
More concretely, the proximal operator $\|\bm y - \bm x\|^2$ is applied to local solvers to impose restrictions on the deviation between global and local solutions in \cite{FedProx}.
Nevertheless, the measure is conservative that only finds an inexact solution near the previous global model $\bm x$.
In this regard, we modify the operator to move the restriction near a local prediction of the current global model so as to accelerate convergence. Then in this paper, the defined personalized model differs from the original local model defined in Expression (\ref{2.fedavg.localupdate}) because of their different initial policies.

\addtolength{\topmargin}{0.02in}
\section{Convergence Analysis}
In this section, we analyze the convergence performance of FedDeper.
To derive the pertinent result, we start by applying some common assumptions.
\begin{assumption}\label{A1}
    $\beta$-smooth: for any $ \bm{y}, \bm{y}^\prime \in \mathbb{R}^d$, there holds
    $$ f_i(\bm{y}) \leq f_i(\bm{y}^\prime) + \langle \nabla f_i(\bm{y}^\prime), \bm{y} -\bm{y}^\prime \rangle + \frac{\beta}{2}\|\bm{y} -\bm{y}^\prime\|^2.$$
\end{assumption}
\begin{assumption}\label{A2}
    Unbiased gradient \& bounded variance: ${g}_{i}$ is unbiased stochastic gradient, i.e.,
    $\mathbb{E}[{g}_{i}] = \nabla f_i$, and its variance is uniformly bounded, i.e.,
    $\mathbb{E}\|{g}_{i} - \nabla f_i\|^2 \leq \varsigma^2$.
\end{assumption}
\begin{assumption}\label{A3}
    Bounded dissimilarity: for any $\bm x \in \mathbb{R}^d$, there exists constants $B^2 \geq 1, G^2 \geq 0$ such that
    $$\frac{1}{n}\sum\nolimits_{i\in \mathcal{N}}\|\nabla f_i(\bm x)\|^2 \leq B^2\|\nabla f(\bm x)\|^2 + G^2.$$
\end{assumption}
All Assumptions \ref{A1}-\ref{A3} are wildly used in existing literatures \cite{koloskova2020unified,SCAFFOLD}.
We now introduce the following to illustrate the convergence of our proposed algorithm\footnote{The full proof is included in https://arxiv.org/pdf/2210.03444.pdf.}.
\begin{theorem}\label{thm.deper} 
    Under Assumptions \ref{A1}-\ref{A3}, by choosing $\rho \leq \eta\beta$, $\eta\tau\beta \leq \min\{\frac{1}{144\tilde{B}^2}, \frac{1}{84\sqrt{2}\sqrt{l_p^1+l_p^2B^2+l_p^3\tilde{B}^2}}\}$, we have
    \begin{equation}\nonumber
        \begin{aligned}
            & \frac{1}{K}\sum\nolimits_{k=0}^{K-1}\mathbb{E}\|\nabla f(\bm x^{k})\|^2 \leq \frac{24\varGamma}{\eta\tau K} + 12 \eta\tau\beta \bigg(4 \tilde{G}^2 + \frac{\varsigma^2}{\tau m}\bigg) 
            \\ & + 24 \eta^2\tau^2\beta^2 \bigg( (1120 + \frac{160}{p}) \tilde{G}^2 +  (1548 + \frac{25}{2p} + \frac{97}{6}
            \\ & + \frac{75}{2}\frac{(1-p)^2}{p^2} ) G^2 + ({330 {{p}}} + \frac{{40}}{m{p}} + \frac{280}{m} + \frac{73}{12}) \frac{\varsigma^2 }{\tau}\bigg) 
            \\ & + 192 \eta^3\tau^3\beta^3 \bigg(3 G^2 +  \frac{\varsigma^2}{\tau}\bigg) + 96 \eta^4\tau^4\beta^4 \bigg(\frac{(3p + {20q} )\varsigma^2}{p\tau}
            \\ & + 12 G^2 \bigg) + 576 \eta^5\tau^5\beta^5 \bigg(4 G^2 + \frac{\varsigma^2}{\tau}\bigg)  + 5760 \eta^6\tau^6\beta^6 \frac{q\varsigma^2}{p\tau}
        \end{aligned}
    \end{equation}
    where $l_p^1:= \frac{15 (1-{{p}})^2}{49{{p}}^2}$, $l_p^2:= 1 + \frac{25}{3136p} + \frac{75(1-p)^2}{3136p^2}$, $l_p^3:= \frac{5}{7}+\frac{5}{49p}$. Besides, $p:= \frac{m}{n}$, $q:= 5+75{{p}} + \frac{15(1-{{p}})^2}{{{p}}}$, $ \varGamma := f(\bm x^0) - f(\bm x^{*}) $, $\tilde{B}^2 := 2B^2(\frac{1}{m}-\frac{1}{n})+1$, $\tilde{G}^2 := 2G^2(\frac{1}{m}-\frac{1}{n})$.
\end{theorem} 
\begin{corollary}\label{thm.rate}
    In terms of Theorem \ref{thm.deper}, by choosing $\eta \leq (\frac{m}{\tau K})^\frac{1}{2}$, we have 
    \begin{equation}\nonumber
        \begin{aligned}
            & \frac{1}{K}\sum\nolimits_{k=0}^{K-1}\mathbb{E}\|\nabla f(\bm x^{k})\|^2 \leq \mathcal{O}\bigg(\frac{\varGamma+m{\tau}\tilde{G}^2 + {\varsigma^2}}{\sqrt{m \tau K}}\bigg) 
            \\ & + \mathcal{O}\bigg(\frac{m\tau \bar G^2 + \varsigma^2}{K}\bigg) + \mathcal{O}\bigg(\frac{(m\tau)^\frac{3}{2}\bar G^2}{K^\frac{3}{2}}\bigg) + \mathcal{O}\bigg(\frac{(m\tau)^2 \bar G^2}{K^2}\bigg)
            \\ & + \mathcal{O}\bigg(\frac{(m\tau)^\frac{5}{2} \bar G^2}{K^\frac{5}{2}}\bigg) + \mathcal{O}\bigg(\frac{(m\tau)^3 \varsigma^2}{K^3}\bigg) 
        \end{aligned}
    \end{equation}
    where $\mathcal{O}$ hides constants including $\beta$, and $\bar G^2:= G^2 + \frac{\varsigma^2}{\tau}$.
\end{corollary}
\begin{remark}
    {\rm 
    Combining with Theorem \ref{thm.deper} and Corollary \ref{thm.rate}, we find the smoothness parameter $\beta$, stochastic variance $\varsigma^2$, the gradient dissimilarity ${G}^2$ and the sampling ratio $p= \frac{m}{n}$ are the dominant factors affecting the convergence rate.
    Note that sampling ratio $p$ contains in the crucial low-order term $12 \eta\tau\beta (4 \tilde{G}^2 + \frac{\varsigma^2}{\tau m})$ with $\tilde{G}^2|_{p=1} = 0$, mainly decides the impact degree of dissimilarity $G^2$ on the training in the dominant convergence rate $\mathcal{O}(\frac{1}{\sqrt{m \tau K}})$.
    Furthermore, penalty constant $\rho$ also implicitly influences the convergence in choosing learning rate $\eta$ due to the precondition $\rho \leq \eta\beta$.
    Besides, the corollary shows appropriately choosing $\eta$ for Theorem \ref{thm.deper} and ignoring high-order terms, the convergence bound can be scaled as $\mathcal{O}(\frac{1}{\sqrt{m \tau K}})$, which meets the sublinear rate similar to works on FedAvg and its variants \cite{SCAFFOLD, FedProx}. }
\end{remark}
\begin{theorem}\label{thm.per}
    Let $\frac{1}{n\tau K}\sum\nolimits_{i,j,k}(\cdot)$ average over all the indexes $i,j,k$, in terms of Theorem \ref{thm.deper}, (i) for any $\lambda \in [\frac{1}{2}, 1)$, we have 
    \begin{equation}\nonumber
        \begin{aligned}
            \frac{1}{n\tau K}\sum\nolimits_{i,j,k}\|\bm v_{i,j}^k - \bm x^*\|^2 \leq \mathcal{O}({\xi^{0}}) + \mathcal{O}(\epsilon),
        \end{aligned}
    \end{equation}
    and (ii) for $\lambda = 1$, we have
    \begin{equation}\nonumber
        \begin{aligned}
            \frac{1}{n\tau K}\sum\nolimits_{i,j,k}\|\bm v_{i,j}^k - \bm x^*\|^2 \leq \mathcal{O}(\epsilon),
        \end{aligned}
    \end{equation}
    where $\mathcal{O}$ hides all constants, $\xi^0 := \frac{1}{n\tau}\sum\nolimits_{i,j}\mathbb{E}\|{\bm v}_{i,j}^0-\bm x^0\|^2 $, and $\epsilon := \frac{1}{K}\sum\nolimits_{k=0}^{K-1}\mathbb{E}\|\nabla f(\bm x^{k})\|^2$.
\end{theorem}
\begin{remark}
    {\rm We here bound the gap between all personalized solutions $(\bm v_{i,j}^k)_{i,j,k}$ and the global optimum $\bm x^*$, and the result in case (i) shows the personalized model converges around $\bm x^*$ in average with the radius $\mathcal{O}({\xi^{0}})$ in terms of the initial distance.
    While in case (ii) we indicate that the behavior of the personalized model degenerates into the original local model: it can converge to the global optimum by choosing an infinitesimal learning rate as shown in Corollary \ref{thm.rate}.
    }
\end{remark}

\section{Performance Evaluation}
We first show the effect of crucial hyper-parameters on FL performance, including penalty $\rho$, mixing rate $\lambda$, local steps $\tau$, and communication rounds $K$.
Then, we conduct aggregated model $\bm x$ performance comparison experiments in both cross-silo ($n=10$) and cross-device ($n=100$) scenarios.
Finally, we investigate the personalized models' performance $(\bm v_i)_{i\in\mathcal{N}}$ in local testing.
\begin{figure}[t]
    \centering
    \subfigure[Effect of $\rho$]{
    \includegraphics[width=0.234\textwidth]{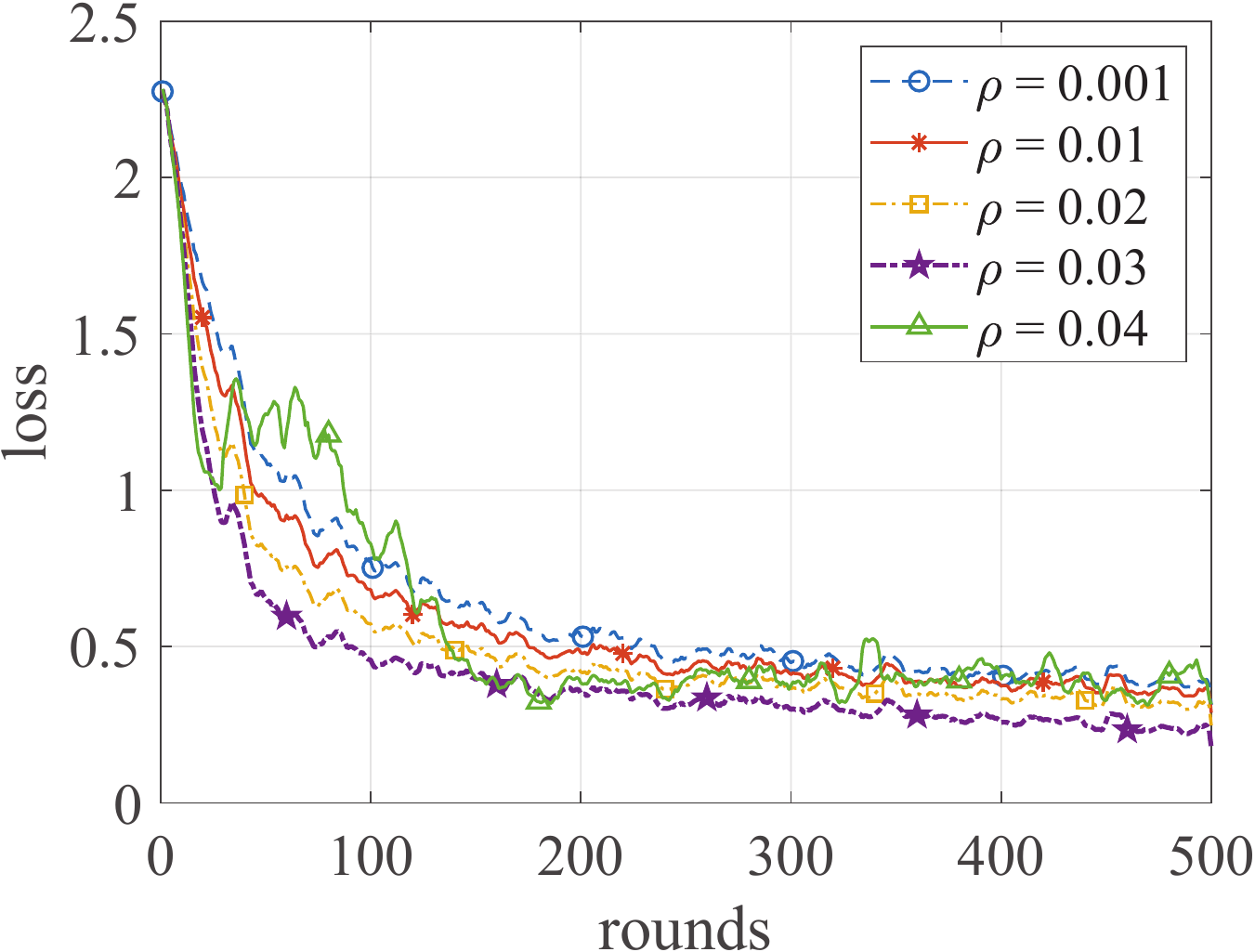}}
    \hfill
    \subfigure[Effect of $\lambda$]{
    \includegraphics[width=0.234\textwidth]{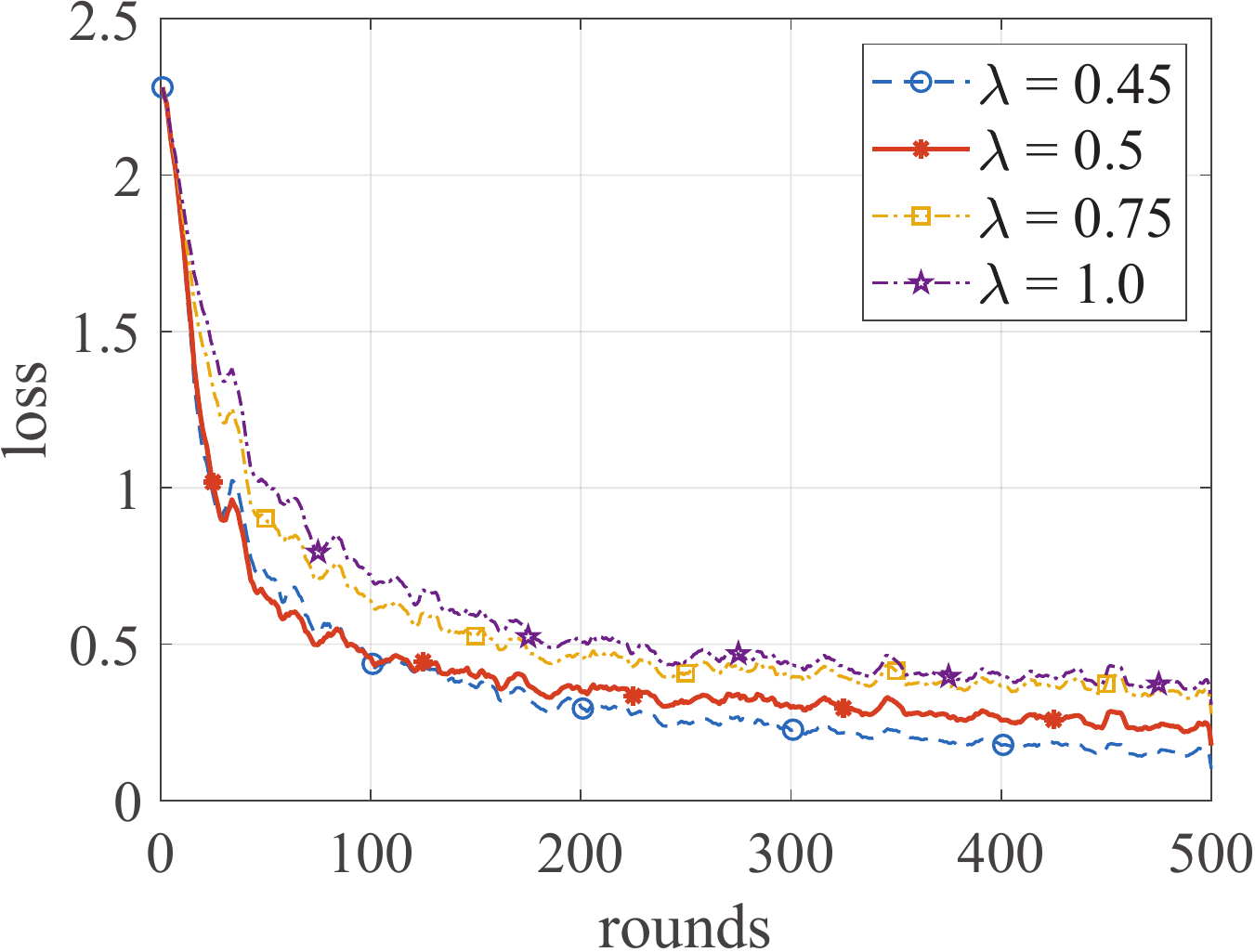}}
    \vfill
    \subfigure[Effect of $\tau$]{
    \includegraphics[width=0.234\textwidth]{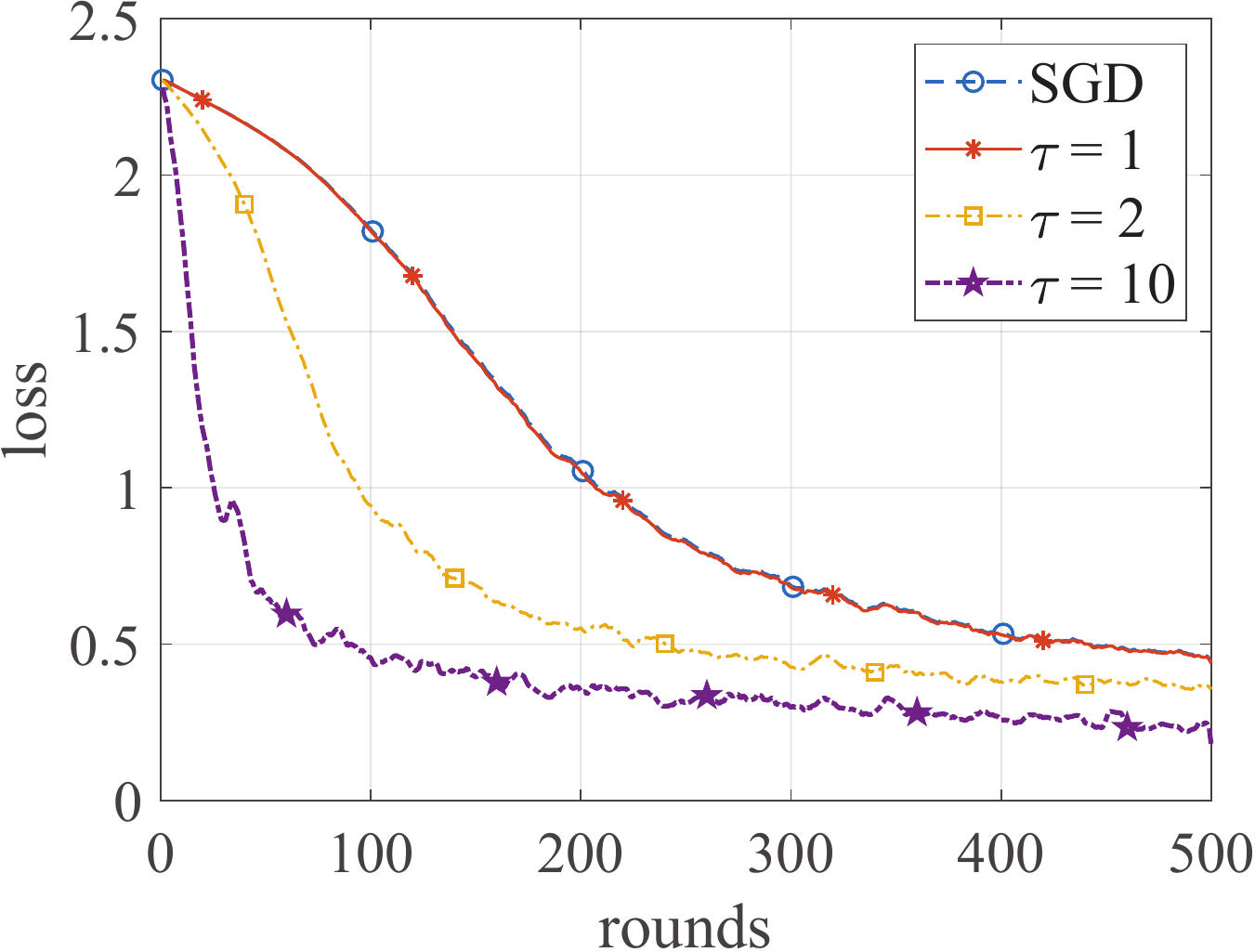}}
    \hfill
    \subfigure[Effect of $K$]{
    \includegraphics[width=0.234\textwidth]{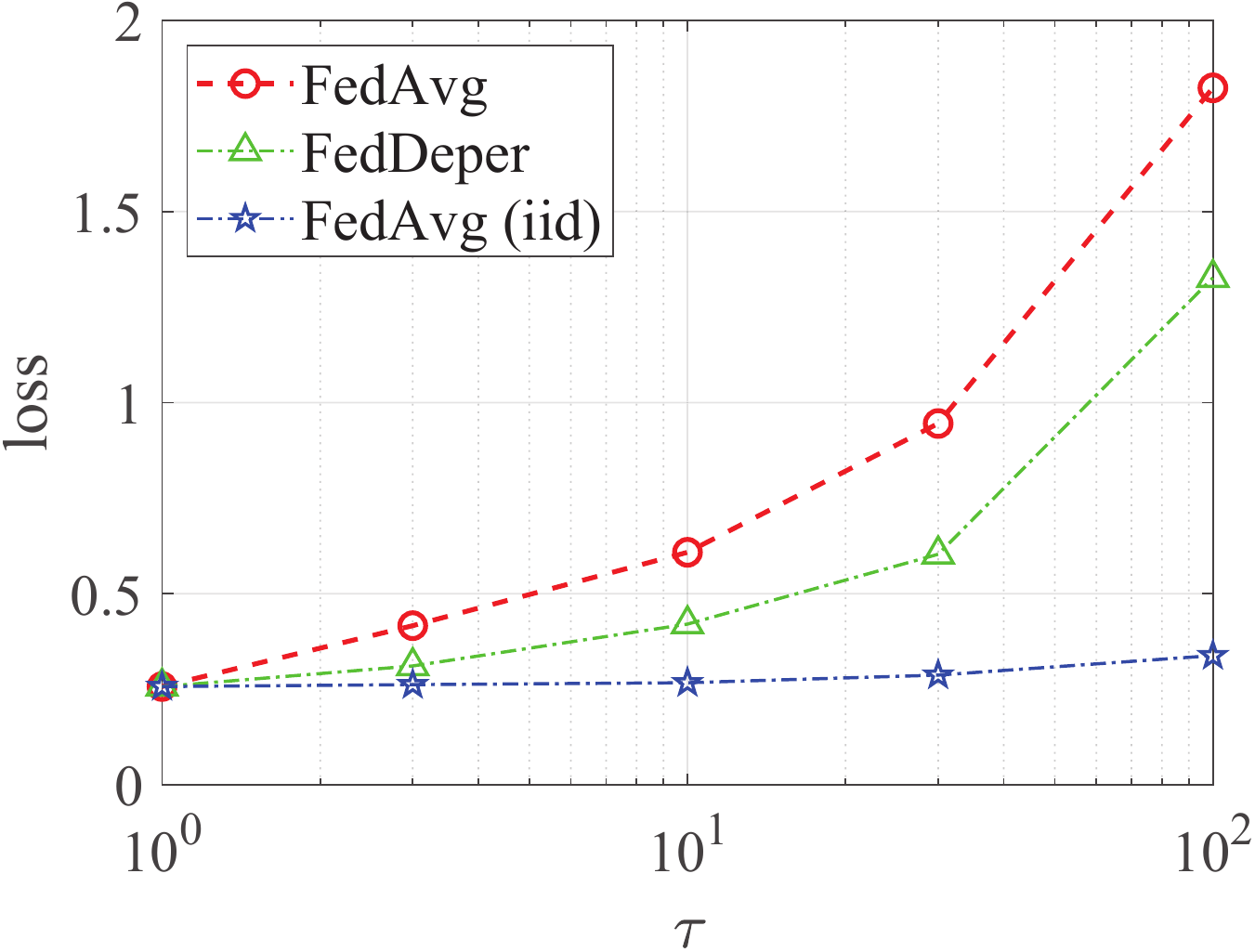}}
    \caption{Effect of Hyper-parameters: (i) using MNIST + MLP and the sampling rate $p= 0.5$ with $n = 10$. (ii) the value of local step $\tau=10$ in (a)(b). (iii) the total communication round $K =500$ in (a)(b)(c) while the total iteration $K\tau = 1500$ in (d). (iv) the penalty $\rho = 0.03$ in (b)(c)(d). (v) the mixing rate $\lambda = 0.5$ in (a)(c)(d).}\label{effe}
\end{figure}
\begin{figure}[t]
    \centering
    \subfigure[MLP \& $m = 5$]{
    \includegraphics[width=0.234\textwidth]{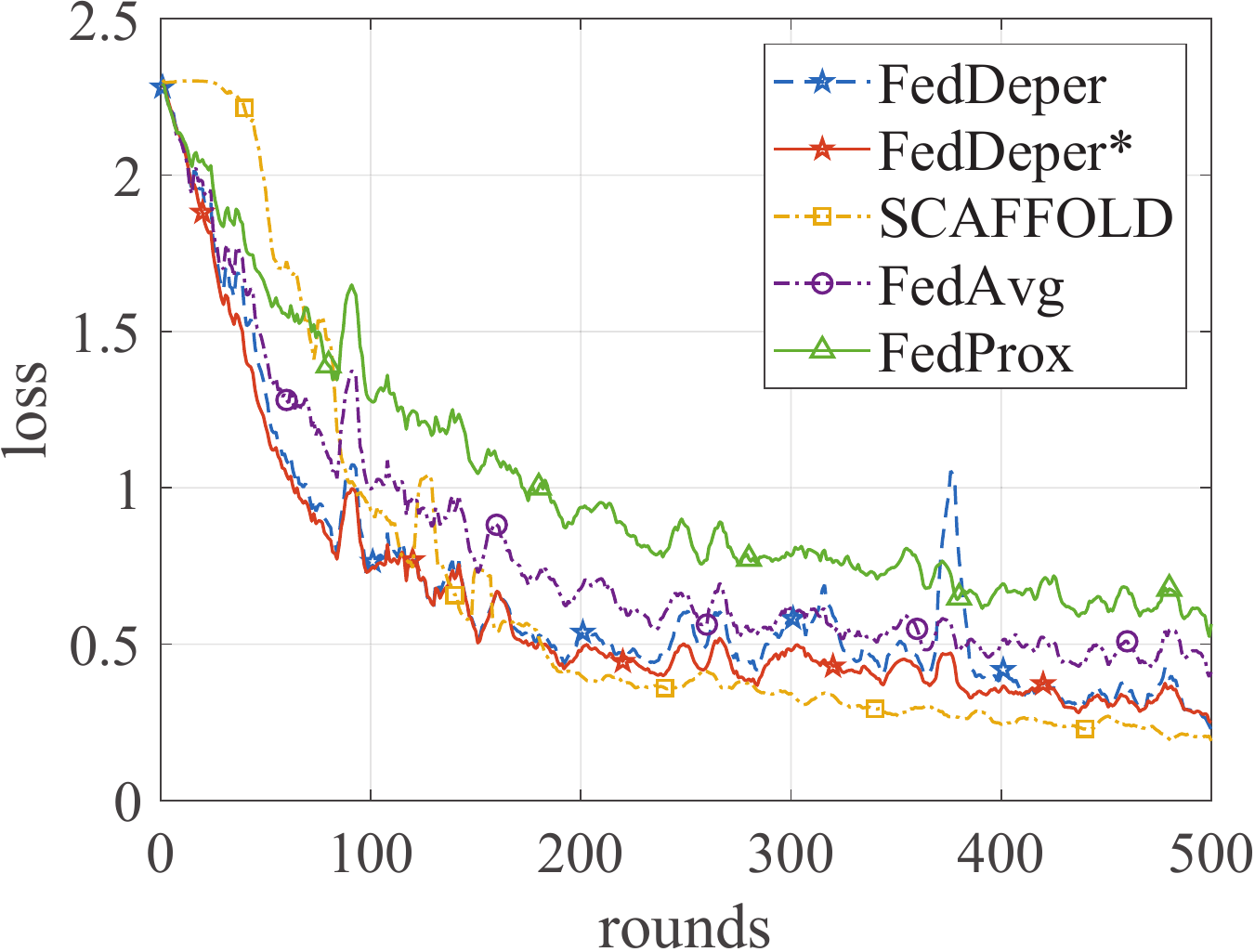}}
    \hfill
    \subfigure[CNN \& $m = 5$]{
    \includegraphics[width=0.234\textwidth]{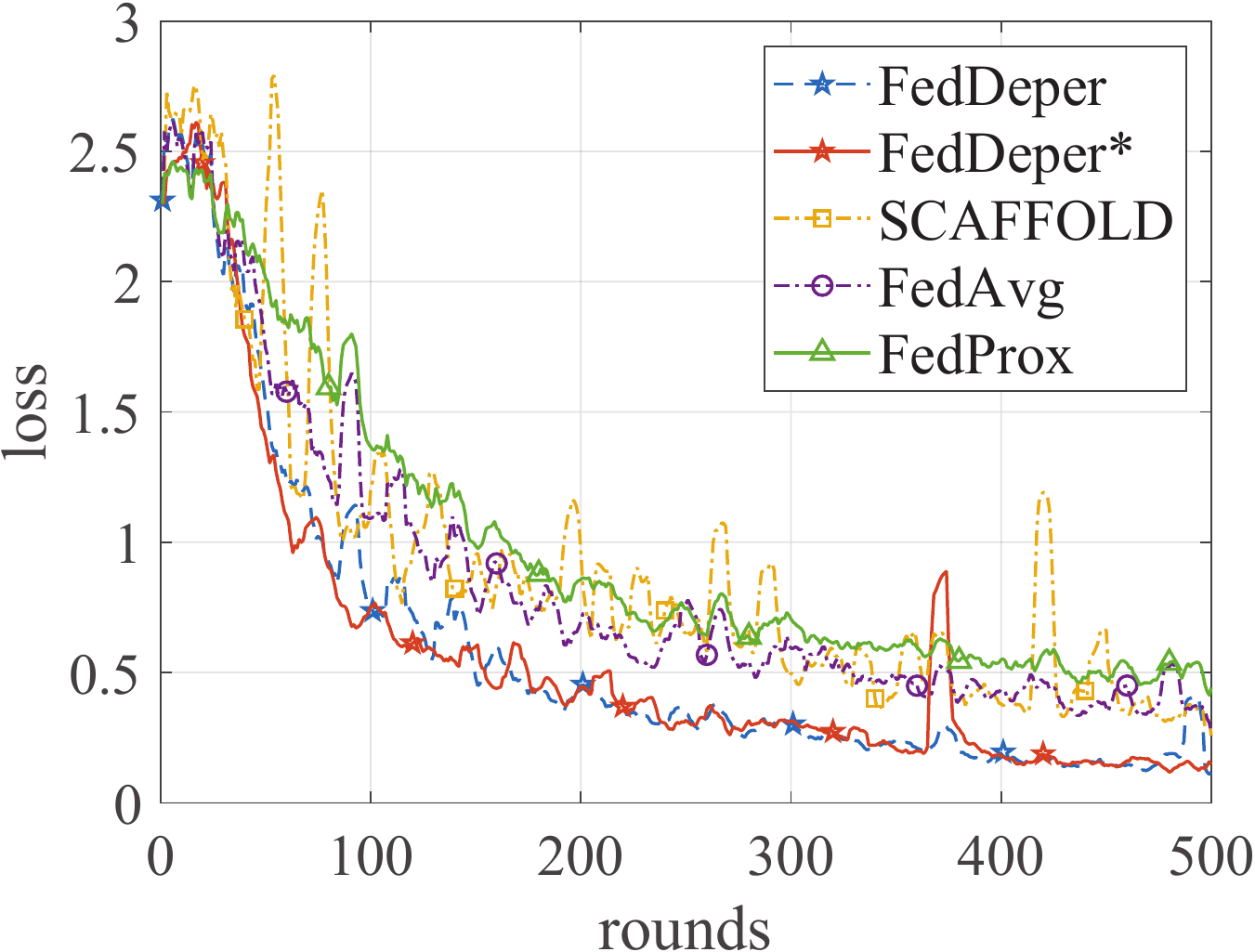}}
    \vfill
    \subfigure[MLP \& $m = 10$]{
    \includegraphics[width=0.234\textwidth]{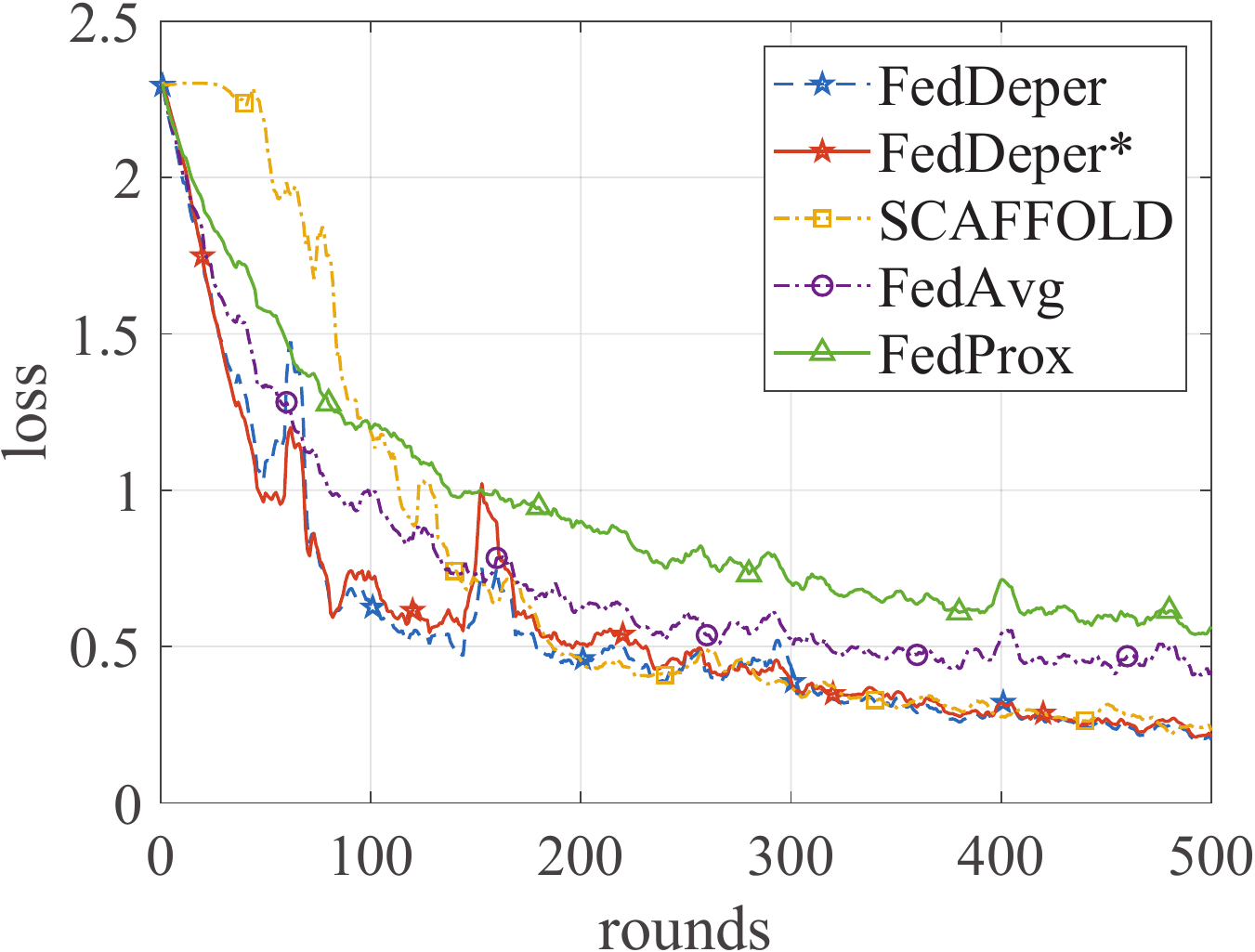}}
    \hfill
    \subfigure[CNN \& $m = 10$]{
    \includegraphics[width=0.234\textwidth]{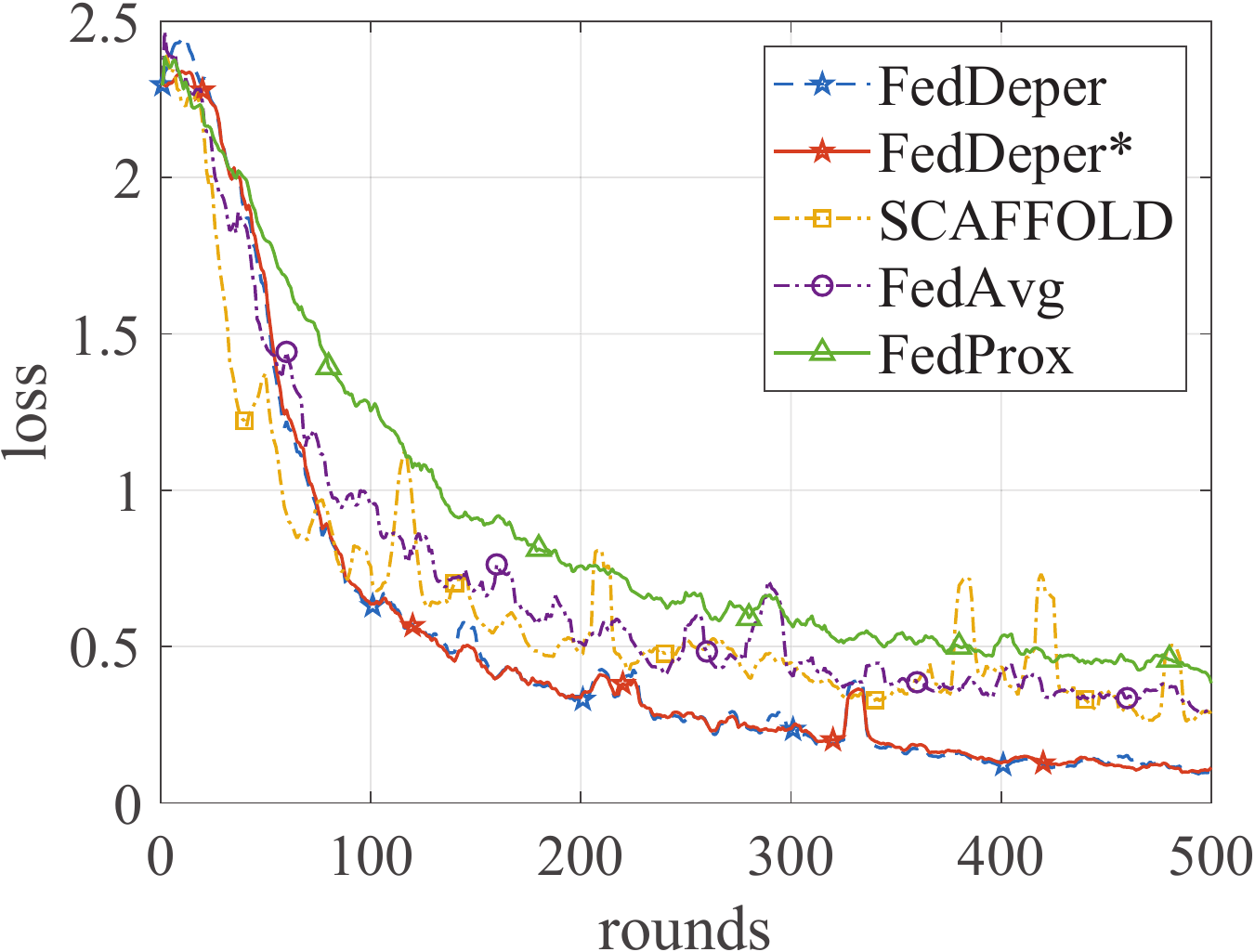}}
    \caption{Convergence Rate Comparison using MNIST in Massive-Device Scenario with $n=100$: (i) for FedDeper, FedAvg, FedProx and SCAFFOLD, the value of local step $\tau=10$, while for FedDeper*, $\tau =5$. (ii) the total communication round $K =500$.
    (iii) the sampling rate $p =$ (a)(b) $0.05$, (c)(d) $0.1$.}\label{mn100}
\end{figure}

\subsection{Experiment Setup}
\noindent\textbf{Basic Settings:}
Each client holds completely heterogeneous raw data generated by non-iid splits as \cite{FedAvg} (sorted data).
Then due to limited bandwidth, the server can only communicate to a subset of clients per round. 
Besides, the learning rate $\eta$ is always set to $0.01$.

\noindent\textbf{Machine Learning Model and Dataset:}
Models: Multilayer Perception (MLP) and Convolutional Neural Network (CNN) with high non-convexity are used as the primary ML model.
Datasets: MNIST and CIFAR-10 as public datasets are used to train ML models with the FL framework.
Model Architectures for Different Datasets: MLP always contains 2 hidden layers with 512 and 256 neurons.
For MNIST, CNN contains 2 convolutional layers with 32 and 64 3$\times$3 filters followed by 2 fully connected layers with 1024 and 512 neurons.
For CIFAR-10, CNN contains 2 convolutional layers with 64 and 128 5$\times$5 filters and 3 fully connected layers with 1024, 512 and 256 neurons.

\noindent\textbf{Baselines:}
We compare FedDeper with the following baselines to evaluate the convergence performance.
\begin{itemize}
    \item FedAvg \cite{FedAvg} is a classical FL method, which is the prototype of FedDeper.
    \item FedProx \cite{FedProx} adds a proximal term as the regularizer to FedAvg for dealing with heterogeneity, which can be regarded as the analogue of our approach.
    \item SCAFFOLD \cite{SCAFFOLD}, the state-of-art method that provably improves the FL performance on non-iid data via cross-client variance reduction but at the expense of double communication overhead, which similarly globalizes local gradients directly with control variables instead of personalized models.
\end{itemize}
Besides, we also provide {FedDeper*} defined as a version of FedDeper with half the local update steps to align the computation costs with baselines.
\begin{table}[t]
    \centering
    \begin{threeparttable}
        \setlength{\tabcolsep}{2.0mm}
        \caption{Testing Accuracy Comparison in Moderate (a) $n = 10$ and Massive (b) $n = 100$ Client Scenarios}
        \label{table_acc}
        \begin{tabular}{c|l|cccccc}
            \toprule
            \multirow{2}{*}{$m$}&\multirow{2}{*}{Method}&
            \multicolumn{2}{c}{(a) MNIST}&\multicolumn{2}{c}{(b) MNIST}&\multicolumn{2}{c}{(b) CIFAR-10}\cr
            \cmidrule(lr){3-4}  \cmidrule(lr){5-6} \cmidrule(lr){7-8} 
            & & MLP& CNN& MLP& CNN& MLP& CNN\cr
            \midrule\multirow{4}{*}{5}
            &Deper  & \textbf{94.92} & \textbf{96.13} & 92.17 & \textbf{95.84} & 48.84 & \textbf{68.19}  \cr
            &Deper* & 94.67 & 95.10 & 92.08 & 95.00 & 47.23 & 67.19  \cr
            &SCAF   & 94.52 & 94.20 & \textbf{93.04} & 90.45 & \textbf{49.17} & 64.11  \cr
            &Avg    & 89.19 & 90.34 & 87.64 & 89.86 & 45.22 & 56.28  \cr
            &Prox   & 86.19 & 86.83 & 82.94 & 87.17 & 44.36 & 49.20  \cr
            
            \midrule\multirow{4}{*}{10}
            &Deper  & 95.28 & \textbf{96.88} & \textbf{93.70} & \textbf{96.55} & 51.05 & \textbf{71.27}  \cr
            &Deper* & 95.23 & 96.58 & 93.42 & 95.11 & 50.66 & 70.87 \cr
            &SCAF   & \textbf{95.32} & 95.64 & 92.86 & 91.39 & \textbf{51.72} & 66.57  \cr
            &Avg    & 89.86 & 93.73 & 87.88 & 90.06 & 49.95 & 60.84  \cr
            &Prox   & 88.77 & 91.11 & 84.02 & 87.36 & 48.53 & 54.80  \cr
            \bottomrule
        \end{tabular}
    \end{threeparttable}
\end{table}

\subsection{Numerical Results}
\begin{figure}[t]
    \centering
    \subfigure[Comparison with baselines]{
    \includegraphics[width=0.234\textwidth]{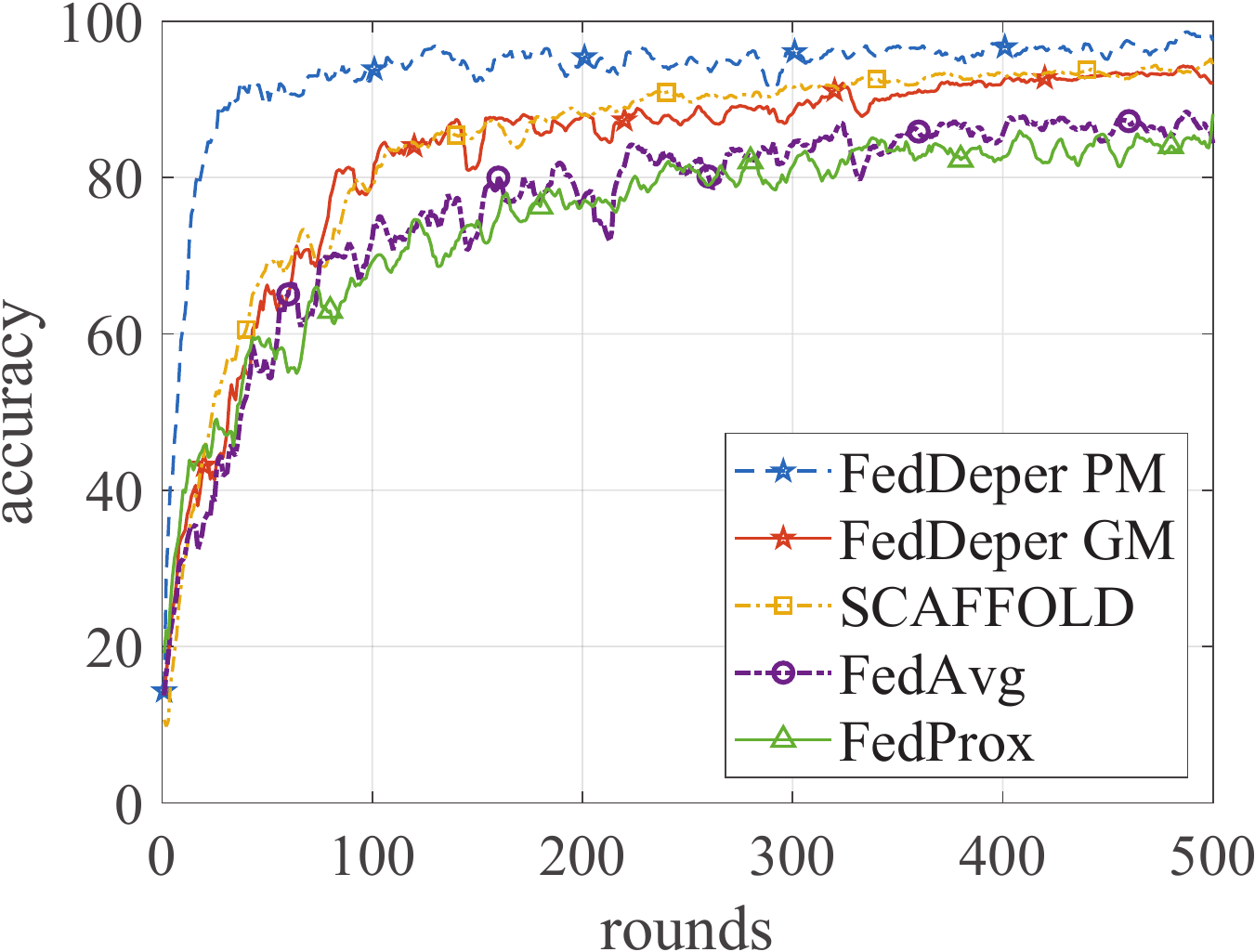}}
    \hfill 
    \subfigure[FedDeper PM Generalization]{
    \includegraphics[width=0.234\textwidth]{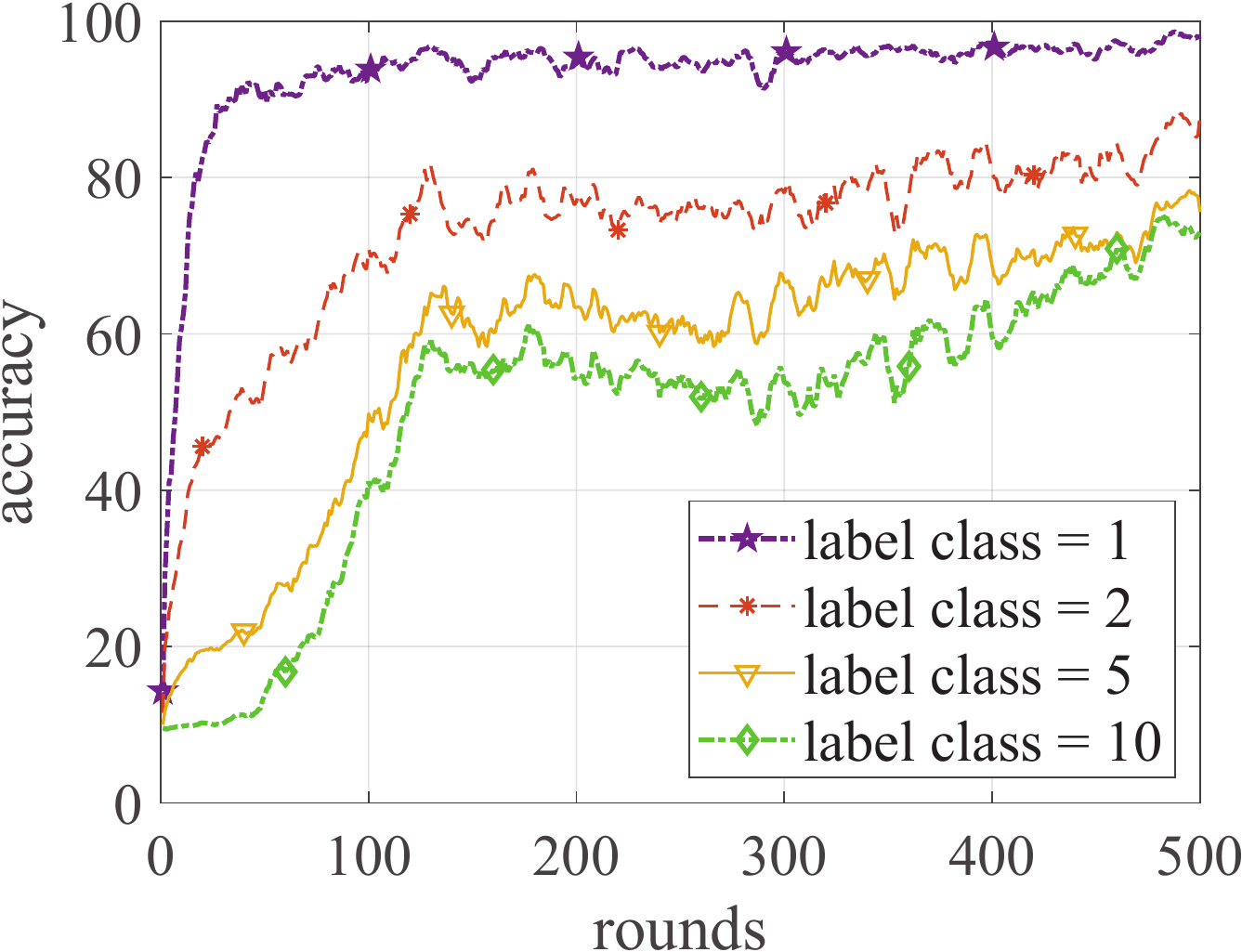}}
    \caption{The top-1 testing accuracy using settings in Fig. \ref{mn100} for personalized models: for (a), PM = personalized model \& GM = global model; for (b), extra-label classes are added incrementally to the local testing dataset.
    }\label{pers}
\end{figure}
\noindent\textbf{Effect of Hyper-parameters:}
The effect of $\rho$, $\lambda$, $\tau$ and $K$ are shown in Fig. \ref{effe}.
In (a), it is observed that $\rho$ is limited to the same order of magnitude as the learning rate $\eta$, and a suitable setting exists to reach the performance upper bound for a particular training environment, e.g., when $\rho=0.03$, the final performance is the best among the five.
In (b), the parameter $\lambda$ yields a similar result as in (a). Besides, the performance with the setting $\lambda = 0.45$ is provided to show that the range of $\lambda$ limited in Remark 1 is sufficient but unnecessary.
In (c), the result illustrates the effectiveness of additional local update steps, namely the convergence speed and final performance (w.r.t. aggregation rounds) improve as $\tau$ increases.
In (d), the reduction of aggregation rounds leads to performance degradation, and the depersonalization mechanism alleviates the objective drift resulting in a better performance of FedDeper than the original FedAvg.

\noindent\textbf{Global Performance on Common Testing Dataset:}
Fig. \ref{mn100} depicts the global training losses varying with communication rounds, which is used to compare the convergence rate of our proposed method with baselines in various settings.
In particular, the proposed FedDeper is carefully tuned to reach its theoretical convergence performance, the proximal constant in FedProx is fixed to $1$, and the FedAvg and SCAFFOLD methods are unmodified and have no extra hyper-parameters.
As shown in the figure, in the massive-device scenario with the low sampling rate, FedDeper has the lowest training loss in all cases except (a), which outperforms almost all baselines and illustrates the effectiveness of FedDeper in convergence acceleration.
Especially, FedDeper performs equally well as the state-of-the-art SCAFFOLD in cases (a)(c) and much better in (b)(d), with only half the communication overhead. Besides, FedDeper is stabler than SCAFFOLD and FedAvg because of the additional regularizer.
Additionally, Table. \ref{table_acc} summarizes the performance of all methods under fixed communication rounds in different settings (including the setting in Fig. \ref{mn100}). 
And the results shows that the proposed FedDeper has significant advantages over baselines across most cases.

\noindent\textbf{Local Performance on Individual Testing Dataset:}
The individual testing dataset on each client is generated by the non-iid splits on the whole testing dataset, which holds only one label class of data samples.
Then Fig. \ref{pers}(a) depicts the averaged testing accuracy overall individual datasets evaluating the local performance of FedDeper and baselines, where FedDeper notably has two models at each client, i.e., global model (GM) and personalized model (PM).
It is illustrated that both FedDeper GM and PM improve the FL performance compared to baselines and that PM converges much faster than GM and baselines experimentally in local testing.
This fact validates the statement in Theorem \ref{thm.per} that personalized models converge around the global model.
Moreover, we randomly add extra label classes to each individual testing dataset to investigate the PM generalization performance via averaged testing accuracy.
As shown in Fig. \ref{pers}(b), PM performs well on the original individual testing dataset with one label class.
Then, as the number of label classes increases, the performance deteriorates since PM is more learned from private data on the client than information (about other classes) through model aggregation. 
Besides, as the communication round increases, overall performance improves in all cases, which shows that model aggregation generalizes PM and enables it to classify other label types missing from the original individual training dataset.

\section{Conclusion}
In this paper, a new FL method FedDeper was proposed to improve performance on non-iid data by reversely using the personalization techniques.
Furthermore, the corresponding convergence of GM and PM was proved and discussed.
Finally, numerical results verified its advantages with the depersonalization mechanism versus existing methods:
(i) FedDeper speeds up the FL convergence.
(ii) FedDeper decouples the FL model in each participator to yield apparent enhancements in local testing.



\footnotesize
\bibliographystyle{IEEEtran}
\bibliography{bibfile}

\begin{thebibliography}{10}
\providecommand{\url}[1]{#1}
\csname url@samestyle\endcsname
\providecommand{\newblock}{\relax}
\providecommand{\bibinfo}[2]{#2}
\providecommand{\BIBentrySTDinterwordspacing}{\spaceskip=0pt\relax}
\providecommand{\BIBentryALTinterwordstretchfactor}{4}
\providecommand{\BIBentryALTinterwordspacing}{\spaceskip=\fontdimen2\font plus
\BIBentryALTinterwordstretchfactor\fontdimen3\font minus
  \fontdimen4\font\relax}
\providecommand{\BIBforeignlanguage}[2]{{%
\expandafter\ifx\csname l@#1\endcsname\relax
\typeout{** WARNING: IEEEtran.bst: No hyphenation pattern has been}%
\typeout{** loaded for the language `#1'. Using the pattern for}%
\typeout{** the default language instead.}%
\else
\language=\csname l@#1\endcsname
\fi
#2}}
\providecommand{\BIBdecl}{\relax}
\BIBdecl

\bibitem{wang2019edge}
X.~Wang, Y.~Han, C.~Wang, Q.~Zhao, X.~Chen, and M.~Chen, ``In-edge ai:
  Intelligentizing mobile edge computing, caching and communication by
  federated learning,'' \emph{IEEE Network}, vol.~33, no.~5, pp. 156--165,
  2019.

\bibitem{zhou2019edge}
Z.~Zhou, X.~Chen, E.~Li \emph{et~al.}, ``Edge intelligence: Paving the last
  mile of artificial intelligence with edge computing,'' \emph{Proceedings of
  the IEEE}, vol. 107, no.~8, pp. 1738--1762, 2019.

\bibitem{dekel2012optimal}
O.~Dekel, R.~Gilad-Bachrach, O.~Shamir, and L.~Xiao, ``Optimal distributed
  online prediction using mini-batches.'' \emph{Journal of Machine Learning
  Research}, vol.~13, no. 165-202, 2012.

\bibitem{FedAvg}
B.~McMahan, E.~Moore, D.~Ramage \emph{et~al.}, ``Communication-efficient
  learning of deep networks from decentralized data,'' in \emph{Artificial
  intelligence and statistics}.\hskip 1em plus 0.5em minus 0.4em\relax PMLR,
  2017, pp. 1273--1282.

\bibitem{koloskova2020unified}
A.~Koloskova, N.~Loizou, S.~Boreiri \emph{et~al.}, ``A unified theory of
  decentralized sgd with changing topology and local updates,'' in
  \emph{International Conference on Machine Learning}.\hskip 1em plus 0.5em
  minus 0.4em\relax PMLR, 2020, pp. 5381--5393.

\bibitem{FAC}
X.~Li, K.~Huang, W.~Yang, S.~Wang, and Z.~Zhang, ``On the convergence of fedavg
  on non-iid data,'' in \emph{International Conference on Learning
  Representations}, 2019.

\bibitem{AFL}
S.~Wang, T.~Tuor, T.~Salonidis, K.~K. Leung, C.~Makaya, T.~He, and K.~Chan,
  ``Adaptive federated learning in resource constrained edge computing
  systems,'' \emph{IEEE Journal on Selected Areas in Communications}, vol.~37,
  no.~6, pp. 1205--1221, 2019.

\bibitem{nguyen2020fast}
H.~T. Nguyen, V.~Sehwag, S.~Hosseinalipour, C.~G. Brinton, M.~Chiang, and H.~V.
  Poor, ``Fast-convergent federated learning,'' \emph{IEEE Journal on Selected
  Areas in Communications}, vol.~39, no.~1, pp. 201--218, 2020.

\bibitem{FedProx}
T.~Li, A.~K. Sahu, M.~Zaheer, M.~Sanjabi, A.~Talwalkar, and V.~Smith,
  ``Federated optimization in heterogeneous networks,'' \emph{Proceedings of
  Machine Learning and Systems}, vol.~2, pp. 429--450, 2020.

\bibitem{SCAFFOLD}
S.~P. Karimireddy, S.~Kale \emph{et~al.}, ``Scaffold: Stochastic controlled
  averaging for federated learning,'' in \emph{International Conference on
  Machine Learning}.\hskip 1em plus 0.5em minus 0.4em\relax PMLR, 2020, pp.
  5132--5143.

\bibitem{TACS}
B.~Luo, W.~Xiao \emph{et~al.}, ``Tackling system and statistical heterogeneity
  for federated learning with adaptive client sampling,'' in \emph{IEEE
  INFOCOM}, 2022, pp. 1739--1748.

\bibitem{li2021ditto}
T.~Li, S.~Hu, A.~Beirami, and V.~Smith, ``Ditto: Fair and robust federated
  learning through personalization,'' in \emph{International Conference on
  Machine Learning}.\hskip 1em plus 0.5em minus 0.4em\relax PMLR, 2021, pp.
  6357--6368.

\bibitem{t2020personalized}
C.~T~Dinh, N.~Tran, and J.~Nguyen, ``Personalized federated learning with
  moreau envelopes,'' \emph{Advances in Neural Information Processing Systems},
  vol.~33, pp. 21\,394--21\,405, 2020.

\bibitem{9766407}
X.~Cao, G.~Sun, H.~Yu, and M.~Guizani, ``Perfed-gan: Personalized federated
  learning via generative adversarial networks,'' \emph{IEEE Internet of Things
  Journal}, pp. 1--1, 2022.

\end{thebibliography}

\onecolumn
\appendix
\section{Proof of Lemmas}
Throughout this paper, we use $\sum\nolimits_{i,j}$ instead of $ \sum\nolimits_{i\in \mathcal{N}}\sum\nolimits_{j \in \{0,1,...,\tau-1\}} $ and use $\mathbb{E}[\cdot]$ to denote the expectation of all random variables in square brackets.
Then to obtain the convergence bound of the algorithm, we introduce the requisite supporting definitions and lemmas as follows.
 
\begin{definition}[Aggregation Gap]
    The aggregation gap between the $k$-th and $k+1$-th rounds is defined as
    \begin{equation}\nonumber
        \begin{aligned}
            \Delta\bm{x}^k := \bm{x}^{k+1}-\bm{x}^k .
        \end{aligned}
    \end{equation}
\end{definition}
 
\begin{definition}[Local Deviation]
    In the $k$-th round, the local deviation between the trained local models and global model is defined as
    \begin{equation}\nonumber
        \zeta^k = \frac{1}{n}\sum\nolimits_{i\in\mathcal{N}}\mathbb{E}\| \bm y_{i,\tau}^k - \bm x^k \|^2.
    \end{equation}
\end{definition}
 
\begin{definition}[Average Deviation]
    In the $k$-th round, the average deviation between all local models and global model is defined as
    \begin{equation}\nonumber
        \psi^k := \frac{1}{n\tau}\sum\nolimits_{i,j}\mathbb{E}\| \bm y_{i,j}^k - \bm x^k \|^2.
    \end{equation}
\end{definition}
 
\begin{definition}[Personalized Deviation]
    In the $k$-th round, the variance between trained (and initial) personalized models and global model is defined as
    \begin{equation}\nonumber
        \begin{aligned}
            &\varphi^k := \frac{1}{n\tau}\sum\nolimits_{i,j}\mathbb{E}\|{\bm v}_{i,j}^k-\bm x^k\|^2, \tilde\varphi^k := \frac{1}{n\tau}\sum\nolimits_{i,j}\mathbb{E}\|{\bm v}_{i,0}^k-\bm x^k\|^2.
        \end{aligned}
    \end{equation}
\end{definition}

\begin{lemma}[Bounded Aggregation Gap]\label{proof.ag}
    The aggregation gap $\Delta\bm{x}^k := \bm{x}^{k+1}-\bm{x}^k$ in any round holds as
    \begin{equation}\nonumber
        \begin{aligned}
            &\mathbb{E}\| \Delta\bm{x}^{k} \|^2 \leq 4(\eta^2\beta^2+\rho^2)\tau^2 \psi^k + 4\rho^2\tau^2 \varphi^k + 4\eta^2\tau^2(\tilde{G}^2+\tilde{B}^2\mathbb{E}\|\nabla f(\bm x^k)\|^2) + \frac{\eta^2\tau\varsigma^2}{m}\\
        \end{aligned}
    \end{equation}    
where $\tilde{B}^2 := 2B^2(\frac{1}{m}-\frac{1}{n})+1$ and $\tilde{G}^2 := 2G^2(\frac{1}{m}-\frac{1}{n})$.
\end{lemma}
\begin{proof} We directly bound the aggregation gap as the following
    \begin{equation}\nonumber
        \begin{aligned}
            \mathbb{E}\| \bm{x}^{k+1} - \bm{x}^{k} \|^2
            =& \mathbb{E}\| \frac{1}{m} \sum\nolimits_{i \in \mathcal{U}^k}(\eta\sum\nolimits_{j=0}^{\tau}{g}_{i,j}^k + \rho \sum\nolimits_{j=0}^{\tau}({\bm v}_{i,j}^k + \bm y_{i,j}^k - 2\bm x^k)) \|^2 \\
            \leq & \frac{4\eta^2\tau}{n}\sum\nolimits_{i,j}\mathbb{E}\| \nabla f_i(\bm y_{i,j}^k) - \nabla f_i(\bm x^k)\|^2 + 4\eta^2\tau^2\mathbb{E}\|\frac{1}{m}\sum\nolimits_{i \in \mathcal{U}^k}\nabla f_i(\bm x^k)\|^2 \\
            & + \frac{4\rho^2\tau}{n}\sum\nolimits_{i,j}\mathbb{E}\| \bm y_{i,j}^k - \bm x^k \|^2 + \frac{4\rho^2\tau}{n}\sum\nolimits_{i,j}\mathbb{E}\|{\bm v}_{i,j}^k - \bm x^k\|^2 + \frac{\eta^2\tau\varsigma^2}{m}\\
            \leq & 4(\eta^2\beta^2+\rho^2)\tau^2 \frac{1}{n\tau}\sum\nolimits_{i,j}\mathbb{E}\| \bm y_{i,j}^k - \bm x^k \|^2 + 4\rho^2\tau^2\frac{1}{n\tau}\sum\nolimits_{i,j}\mathbb{E}\|{\bm v}_{i,j}^k-\bm x^k\|^2\\
            & + 4\eta^2\tau^2(\mathbb{E}\|\nabla f(\bm x^k)\|^2 + 2(\frac{1}{m}-\frac{1}{n})\frac{1}{n}\sum\nolimits_{i}\|\nabla f_i(\bm x)\|^2) + \frac{\eta^2\tau\varsigma^2}{m}\\
            \leq & 4(\eta^2\beta^2+\rho^2)\tau^2 \frac{1}{n\tau}\sum\nolimits_{i,j}\mathbb{E}\| \bm y_{i,j}^k - \bm x^k \|^2 + 4\rho^2\tau^2\frac{1}{n\tau}\sum\nolimits_{i,j}\mathbb{E}\|{\bm v}_{i,j}^k-\bm x^k\|^2\\
            & + 4\eta^2\tau^2(\mathbb{E}\|\nabla f(\bm x^k)\|^2 + 2(\frac{1}{m}-\frac{1}{n})(B^2\mathbb{E}\|\nabla f(\bm x)\|^2+G^2)) + \frac{\eta^2\tau\varsigma^2}{m}\\
            =& 4(\eta^2\beta^2+\rho^2)\tau^2 \psi^k + 4\rho^2\tau^2 \varphi^k + 4\eta^2\tau^2(\tilde{G}^2+\tilde{B}^2\mathbb{E}\|\nabla f(\bm x^k)\|^2) + \frac{\eta^2\tau\varsigma^2}{m},\\
        \end{aligned}
    \end{equation}
    where the first inequality holds in \cite{FAC, SCAFFOLD} for separating mean and variance. Then we use triangle inequalities repeatedly to derive the result based on defined variables.
\end{proof}
  
\begin{lemma}[Bounded Local Deviation]\label{proof.pd}
    For any aggregation round $k$, the local deviation $\zeta^{k}$ is bounded by
    \begin{equation}\nonumber
        \begin{aligned}
            \zeta^k := \frac{1}{n}\sum\nolimits_{i\in\mathcal{N}}\mathbb{E}\| \bm y_{i,\tau}^k - \bm x^k \|^2 \leq 4(\eta^2\beta^2+\rho^2)\tau^2 \psi^k + 4\rho^2\tau^2 \varphi^k + 4\eta^2\tau^2(G^2+B^2\mathbb{E}\|\nabla f(\bm x^k)\|^2) + \eta^2\tau\varsigma^2
        \end{aligned}
    \end{equation}
\end{lemma}
\begin{proof} Applying the triangle inequality, we have
    \begin{equation}\nonumber
        \begin{aligned}
            \zeta^k =& \frac{1}{n}\sum\nolimits_{i}\mathbb{E}\| \bm y_{i,\tau}^k - \bm x^k \|^2 \\
            \leq & \frac{1}{n}\sum\nolimits_{i}\mathbb{E}\| \eta\sum\nolimits_{j=0}^{\tau}{g}_{i,j}^k + \rho \sum\nolimits_{j=0}^{\tau}(\bm{v}_{i,j}^k + \bm y_{i,j}^k - 2\bm x^k) \|^2 \\
            \leq & \frac{4\eta^2\tau}{n}\sum\nolimits_{i,j}\mathbb{E}\|\nabla f_i(\bm{y}_{i,j}^k) - \nabla f_i({\bm{x}^k})\|^2 + \frac{4\eta^2\tau}{n}\sum\nolimits_{i,j}\mathbb{E}\|\nabla f_i(\bm{x}^k)\|^2 \\ 
            & +\frac{4\rho^2\tau}{n}\sum\nolimits_{i,j}\mathbb{E}\|\bm y_{i,j}^k - \bm x^k\|^2 + \frac{4\rho^2\tau}{n}\sum\nolimits_{i,j}\mathbb{E}\|\bm{v}_{i,j}^k - \bm x^k\|^2 + \eta^2\tau\varsigma^2 \\
            \leq & 4(\eta^2\beta^2+\rho^2)\tau^2 \psi^k + 4\rho^2\tau^2 \varphi^k + 4\eta^2\tau^2(G^2+B^2\mathbb{E}\|\nabla f(\bm x^k)\|^2) + \eta^2\tau\varsigma^2\\
        \end{aligned}
    \end{equation}
\end{proof}

\begin{lemma}[Bounded Average Deviation]\label{proof.apd}
    For any aggregation round $k$ and any $\eta,\rho$ satisfied $\eta^2\beta^2\leq\frac{(1-\rho)^2}{12\tau(\tau-1)}$, the average deviation $\psi^k := \frac{1}{n\tau}\sum\nolimits_{i,j}\mathbb{E}\| \bm y_{i,j}^k - \bm x^k \|^2$ is bounded by
    \begin{equation}\nonumber
        \begin{aligned}
            \psi^k \leq & 12\rho^2(1-\rho)^2\tau^2\tilde{\varphi}^k + 24(1-\rho)^2 \rho^2\eta^2\tau^4\beta^2\varphi^k 
            \\ & + (1 + 4\rho^2\tau^2) 6(1-\rho)^2\eta^2\tau^2(G^2+B^2\mathbb{E}\|\nabla f(\bm x^k)\|^2) + (1 + 3\rho^2\tau^2)2(1-\rho)^2\eta^2\tau\varsigma^2.\\
        \end{aligned}
    \end{equation}
\end{lemma}
\begin{proof}We first provide the one-step result and then unroll it to get the upper-bound of $ \psi^k $:
    \begin{equation}\nonumber
        \begin{aligned}
            \mathbb{E}\| \bm y_{i,j}^k - \bm x^k \|^2 \leq & \mathbb{E}\| \bm y_{i,j-1}^k - \eta \nabla f_i(\bm y_{i,j-1}^k) - \rho({\bm v}_{i,j-1}^k + \bm y_{i,j-1}^k - 2\bm x^k) - \bm x^k \|^2 + \eta^2\varsigma^2\\
            \leq & (1+\frac{1}{\tau-1})(1-\rho)^2\mathbb{E}\| \bm y_{i,j-1}^k - \bm x^k \|^2 + \tau \mathbb{E}\| \eta \nabla f_i(\bm y_{i,j-1}^k) + \rho({\bm v}_{i,j-1}^k - \bm x^k) \|^2 + \eta^2\varsigma^2\\
            \leq & ((1+\frac{1}{\tau-1})(1-\rho)^2+3\eta^2\beta^2\tau)\mathbb{E}\| \bm y_{i,j-1}^k - \bm x^k \|^2 + 3\eta^2\tau \mathbb{E}\|\nabla f_i(\bm x^k)\|^2 + 3\rho^2\tau \mathbb{E}\|{\bm v}_{i,j-1}^k - \bm x^k\|^2 + \eta^2\varsigma^2\\
        \end{aligned}
    \end{equation}
    Averaging the above over index $i$, we have
    \begin{equation}\nonumber
        \begin{aligned}
            \frac{1}{n}&\sum\nolimits_i\mathbb{E}\| \bm y_{i,j}^k - \bm x^k \|^2 
            \\ \leq & ((1+\frac{1}{\tau-1})(1-\rho)^2+3\eta^2\beta^2\tau)\frac{1}{n}\sum\nolimits_i\mathbb{E}\| \bm y_{i,j-1}^k - \bm x^k \|^2 + 3\eta^2\tau\frac{1}{n}\sum\nolimits_i\mathbb{E}\|\nabla f_i(\bm x^k)\|^2 \\
            & + 3\rho^2\tau \frac{1}{n}\sum\nolimits_i\mathbb{E}\|{\bm v}_{i,j-1}^k - \bm x^k\|^2 + \eta^2\varsigma^2\\
            \leq & ((1+\frac{1}{\tau-1})(1-\rho)^2+3\eta^2\beta^2\tau)\frac{1}{n}\sum\nolimits_i\mathbb{E}\| \bm y_{i,j-1}^k - \bm x^k \|^2 + 3\eta^2\tau\frac{1}{n}\sum\nolimits_i\mathbb{E}\|\nabla f_i(\bm x^k)\|^2 \\
            & + 3\rho^2\tau \sup\nolimits_j\frac{1}{n}\sum\nolimits_i\mathbb{E}\|{\bm v}_{i,j}^k - \bm x^k\|^2 + \eta^2\varsigma^2\\
            \leq & \left(\sum\nolimits_{j}((1+\frac{1}{\tau-1})(1-\rho)^2+3\eta^2\beta^2\tau)^j\right)\left(3\eta^2\tau\frac{1}{n}\sum\nolimits_i\mathbb{E}\|\nabla f_i(\bm x^k)\|^2 + 3\rho^2\tau \sup\nolimits_j\frac{1}{n}\sum\nolimits_i\mathbb{E}\|{\bm v}_{i,j}^k - \bm x^k\|^2 + \eta^2\varsigma^2\right)\\
        \end{aligned}
    \end{equation}
    To make $\sum\nolimits_{j}((1+\frac{1}{\tau-1})(1-\rho)^2+3\eta^2\beta^2\tau)^j$ have the linear growth, we should check the inequality
    \begin{equation}\nonumber
        \begin{aligned}
            \sum\nolimits_{j}(1+\frac{1+\theta}{\tau-1})^j = \frac{(1+\frac{1+\theta}{\tau-1})^\tau-1}{\frac{1+\theta}{\tau-1}} := \frac{(1+\omega)^\tau-1}{\omega} \leq \frac{e^{\omega\tau}-1}{\omega} \leq C_\omega \cdot \tau
        \end{aligned}
    \end{equation}
    Where we set $\omega(\tau) := \frac{1+\theta}{\tau-1}$ temporarily and then calculate the derivative of $ g_\omega(\tau) := \frac{e^{\omega\tau}-1}{\omega} - C_\omega \cdot \tau $ to yield the tight bound of $C_\omega$,
    \begin{equation}\nonumber
        \begin{aligned}
            g^\prime =& (1+(\frac{\tau}{\omega}-\frac{1}{\omega^2})\omega^\prime)e^{\omega\tau} + \frac{\omega^\prime}{\omega^2} - C_\omega 
            = (\frac{1}{1+\theta}-\frac{1}{\tau-1})e^{\frac{\tau}{\tau-1}(1+\theta)} - \frac{1}{1+\theta} - C_\omega \leq 0 \\
        \end{aligned}
    \end{equation}
    Let $\theta = \frac{1}{4}$ which implies $\eta^2\beta^2\leq\frac{(1-\rho)^2}{12\tau(\tau-1)}$, we find the constant $C_\omega \leq 2$ and then plug it back to the inequality,
    \begin{equation}\nonumber
        \begin{aligned}
            \frac{1}{n}\sum\nolimits_i&\mathbb{E}\| \bm y_{i,j}^k - \bm x^k \|^2 \leq  2(1-\rho)^2\tau\left(3\eta^2\tau\frac{1}{n}\sum\nolimits_i\mathbb{E}\|\nabla f_i(\bm x^k)\|^2 + 3\rho^2\tau \sup\nolimits_j\frac{1}{n}\sum\nolimits_i\mathbb{E}\|{\bm v}_{i,j}^k - \bm x^k\|^2 + \eta^2\varsigma^2\right)\\
            \leq & 6\eta^2(1-\rho)^2\tau^2\frac{1}{n}\sum\nolimits_i\mathbb{E}\|\nabla f_i(\bm x^k)\|^2 + 6\rho^2(1-\rho)^2\tau^2 \sup\nolimits_j\frac{1}{n}\sum\nolimits_i\mathbb{E}\|{\bm v}_{i,j}^k - \bm x^k\|^2 + 2(1-\rho)^2\eta^2\tau\varsigma^2\\
            \leq & 6\eta^2(1-\rho)^2\tau^2\frac{1}{n}\sum\nolimits_i\mathbb{E}\|\nabla f_i(\bm x^k)\|^2 +  12\rho^2(1-\rho)^2\tau^2\frac{1}{n}\sum\nolimits_i\mathbb{E}\|{\bm v}_{i,0}^k-\bm x^k\|^2 
            \\ & + 24(1-\rho)^2 \rho^2\eta^2\tau^4\beta^2\varphi^k + 24\rho^2(1-\rho)^2\tau^2 \eta^2\tau^2(G^2+B^2\mathbb{E}\|\nabla f(\bm x^k)\|^2) + (1 + 3\rho^2\tau^2)2(1-\rho)^2\eta^2\tau\varsigma^2\\
        \end{aligned}
    \end{equation}
    The last inequality holds because we have the following
    \begin{equation}\nonumber
        \begin{aligned}
            \sup\nolimits_j\frac{1}{n}\sum\nolimits_i\mathbb{E}\|{\bm v}_{i,j}^k - \bm x^k\|^2
            =& \sup\nolimits_j\frac{1}{n}\sum\nolimits_i\mathbb{E}\|{\bm v}_{i,0}^k - {\eta}\sum\nolimits_{j^\prime = 0}^{j-1}g_i({\bm v}_{i,j^\prime}^k) -\bm x^k\|^2 \\
            \leq & 2\frac{1}{n}\sum\nolimits_i\mathbb{E}\|{\bm v}_{i,0}^k-\bm x^k\|^2 + 2\sup\nolimits_j\frac{1}{n}\sum\nolimits_i\mathbb{E}\|{\eta}\sum\nolimits_{j^\prime = 0}^{j-1}\nabla f_i({\bm v}_{i,j^\prime}^k)\|^2 + \eta^2\tau\varsigma^2\\
            \leq & 2\frac{1}{n}\sum\nolimits_i\mathbb{E}\|{\bm v}_{i,0}^k-\bm x^k\|^2 + 2\eta^2\tau^2\frac{1}{n\tau}\sum\nolimits_{i,j}\mathbb{E}\|\nabla f_i({\bm v}_{i,j}^k)\|^2 + \eta^2\tau\varsigma^2\\
            \leq & 2\frac{1}{n}\sum\nolimits_i\mathbb{E}\|{\bm v}_{i,0}^k-\bm x^k\|^2 +4\eta^2\tau^2\beta^2\varphi^k +4\eta^2\tau^2(G^2+B^2\mathbb{E}\|\nabla f(\bm x^k)\|^2) + \eta^2\tau\varsigma^2\\
        \end{aligned}
    \end{equation}
    Indeed 
    \begin{equation}\nonumber
        \begin{aligned}
            \frac{1}{n\tau}\sum\nolimits_{i,j}\mathbb{E}\|\nabla f_i({\bm v}_{i,j}^k)\|^2
            =& \frac{1}{n\tau}\sum\nolimits_{i,j}\mathbb{E}\|\nabla f_i({\bm v}_{i,j}^k)-\nabla f_i(\bm x^k)+\nabla f_i(\bm x^k)\|^2 \\
            \leq & 2\frac{1}{n\tau}\sum\nolimits_{i,j}\mathbb{E}\|\nabla f_i({\bm v}_{i,j}^k)-\nabla f_i(\bm x^k)\|^2 +2\frac{1}{n\tau}\sum\nolimits_{i,j}\mathbb{E}\|\nabla f_i(\bm x^k)\|^2 \\
        \end{aligned}
    \end{equation}
    Finally, we summarize as follows
    \begin{equation}\nonumber
        \begin{aligned}
            \frac{1}{n\tau}\sum\nolimits_{i,j}&\mathbb{E}\| \bm y_{i,j}^k - \bm x^k \|^2 \leq 6\eta^2(1-\rho)^2\tau^2\frac{1}{n}\sum\nolimits_i\mathbb{E}\|\nabla f_i(\bm x^k)\|^2 +  12\rho^2(1-\rho)^2\tau^2\frac{1}{n}\sum\nolimits_i\mathbb{E}\|{\bm v}_{i,0}^k-\bm x^k\|^2 
            \\ & + 24(1-\rho)^2 \rho^2\eta^2\tau^4\beta^2\varphi^k + 24\rho^2(1-\rho)^2\eta^2\tau^4(G^2+B^2\mathbb{E}\|\nabla f(\bm x^k)\|^2) + (1 + 3\rho^2\tau^2)2(1-\rho)^2\eta^2\tau\varsigma^2,\\
        \end{aligned}
    \end{equation}
    which implies the result in the lemma.
\end{proof}

\begin{lemma}[Bounded Personalization Deviation]\label{proof.cpd}
    For any aggregation round $k$, the personalized deviation $\varphi^{k}$ is bounded by
    \begin{equation}\nonumber
        \begin{aligned}
            \varphi^k \leq 2\tilde{\varphi}^k + 4\eta^2\tau^2\beta^2\varphi^k + 2\eta^2\tau^2(G^2+B^2\mathbb{E}\|\nabla f(\bm x^k)\|^2) + \eta^2\tau\varsigma^2
        \end{aligned}
    \end{equation}
    where for any constant $c > 0$ we also have the bounded initial one 
    \begin{equation}\nonumber
        \begin{aligned}
            \tilde\varphi^k :=& \frac{1}{n\tau}\sum\nolimits_{i,j}\mathbb{E}\|{\bm v}_{i,0}^k-\bm x^k\|^2\\
            \leq & ((1-{{p}})(1+c)+2{{p}}(1-\lambda)^2)\tilde{\varphi}^{k-1} + (1+7{{p}})\mathbb{E}\|\Delta \bm x^{k-1}\|^2+(1-{{p}})\frac{1}{c}\mathbb{E}\|\mathbb{E}[\Delta \bm x^{k-1}]\|^2\\
            & +8{{p}}(1-\lambda)^2\eta^2\tau^2\beta^2\varphi^{k-1}+ 8{{p}}(1-\lambda)^2\eta^2\tau^2(G^2+B^2\mathbb{E}\|\nabla f(\bm x^{k-1})\|^2)+8{{p}}\lambda^2\zeta^{k-1} + {{p}}(1-\lambda)^2\eta^2\tau\varsigma^2\\
        \end{aligned}
    \end{equation}
\end{lemma}
\begin{proof} We directly process the variable with the triangle inequality
    \begin{equation}\label{proof.cpd.f1}
        \begin{aligned}
            \varphi^k =& \frac{1}{n\tau}\sum\nolimits_{i,j}\mathbb{E}\|{\bm v}_{i,j}^k-\bm x^k\|^2 \\
            =& \frac{1}{n\tau}\sum\nolimits_{i,j}\mathbb{E}\|{\bm v}_{i,0}^k - {\eta}\sum\nolimits_{j^\prime = 0}^{j-1}\nabla f_i({\bm v}_{i,j^\prime}^k) -\bm x^k\|^2 + \frac{(1+\tau)}{2}\eta^2\varsigma^2 \\
            \leq & 2\underbrace{\frac{1}{n\tau}\sum\nolimits_{i,j}\mathbb{E}\|{\bm v}_{i,0}^k-\bm x^k\|^2}_{T_1:=\tilde\varphi^k} + 2\underbrace{\frac{1}{n\tau}\sum\nolimits_{i,j}\mathbb{E}\|{\eta}\sum\nolimits_{j^\prime = 0}^{j-1}\nabla f_i({\bm v}_{i,j^\prime}^k)\|^2}_{T_2} + \eta^2\tau\varsigma^2 \\
        \end{aligned}
    \end{equation}
    Then we bound term $T_1$ and $T_2$ respectively. Firstly, we assert $T_1:=\tilde\varphi^k $ can be bounded by
    \begin{equation}\label{proof.cpd.T1}
        \begin{aligned}
            \tilde\varphi^k\leq & ((1-{{p}})(1+c)+2{{p}}(1-\lambda)^2)\tilde{\varphi}^{k-1} + (1+7{{p}})\mathbb{E}\|\Delta \bm x^{k-1}\|^2+(1-{{p}})\frac{1}{c}\mathbb{E}\|\mathbb{E}[\Delta \bm x^{k-1}]\|^2\\
            & +8{{p}}(1-\lambda)^2\eta^2\tau^2\beta^2\varphi^{k-1}+ 8{{p}}(1-\lambda)^2\eta^2\tau^2(G^2+B^2\mathbb{E}\|\nabla f(\bm x^{k-1})\|^2)+8{{p}}\lambda^2\zeta^{k-1} + {{p}}(1-\lambda)^2\eta^2\tau\varsigma^2\\
        \end{aligned}
    \end{equation}
    Sequentially, we calculate $T_2$ as follows
    \begin{equation}\label{proof.cpd.T2}
        \begin{aligned}
            T_2 =& \frac{1}{n\tau}\sum\nolimits_{i,j}\mathbb{E}\|{\eta}\sum\nolimits_{j^\prime = 0}^{j-1}\nabla f_i({\bm v}_{i,j^\prime}^k)\|^2 \\
            \leq & 2\frac{1}{n\tau}\sum\nolimits_{i,j}\mathbb{E}\|{\eta}\sum\nolimits_{j^\prime = 0}^{j-1}(\nabla f_i({\bm v}_{i,j^\prime}^k)-\nabla f_i(\bm x^k))\|^2 + 2{\eta^2}\frac{1}{n\tau}\sum\nolimits_{i,j}j^2\mathbb{E}\|\nabla f_i(\bm x^k)\|^2 \\
            \leq & 2\eta^2\tau^2\frac{1}{n\tau}\sum\nolimits_{i,j}\mathbb{E}\|\nabla f_i({\bm v}_{i,j}^k)-\nabla f_i(\bm x^k)\|^2 + 2\eta^2\tau^2\frac{\tau(\tau-1)(2\tau-1)}{6\tau^3}\frac{1}{n}\sum\nolimits_{i}\mathbb{E}\|\nabla f_i(\bm x^k)\|^2 \\
            \leq & 2\eta^2\tau^2\beta^2\varphi^k + \eta^2\tau^2(G^2+B^2\mathbb{E}\|\nabla f(\bm x^k)\|^2) \\
        \end{aligned}
    \end{equation}
    Having established (\ref{proof.cpd.T1}) and (\ref{proof.cpd.T2}), we can now plug them back to (\ref{proof.cpd.f1}) and derive the bound asserted in Lemma \ref{proof.cpd}:
    \begin{equation}\nonumber
        \begin{aligned}
            \varphi^k \leq 2\tilde{\varphi}^k + 4\eta^2\tau^2\beta^2\varphi^k + 2\eta^2\tau^2(G^2+B^2\mathbb{E}\|\nabla f(\bm x^k)\|^2) + \eta^2\tau\varsigma^2
        \end{aligned}
    \end{equation}
    Now we prove the assertion (\ref{proof.cpd.T1}).
    \begin{equation}\nonumber
        \begin{aligned}
            \tilde\varphi^k=& T_1 = \frac{1}{n\tau}\sum\nolimits_{i,j}\mathbb{E}\|{\bm v}_{i,0}^k-\bm x^k\|^2 \\
            =& (1-{{p}})\frac{1}{n}\sum\nolimits_{i}\mathbb{E}\|{\bm v}_{i,0}^{k-1}-\bm x^k\|^2 + {{p}}\frac{1}{n}\sum\nolimits_{i}\mathbb{E}\|(1-\lambda){\bm v}_{i,\tau}^{k-1}+\lambda\bm y_{i,\tau}^{k-1}-\bm x^k\|^2 \\
            \leq& (1-{{p}})\frac{1}{n}\sum\nolimits_{i}\underbrace{\mathbb{E}(\|{\bm v}_{i,0}^{k-1}-\bm x^{k-1}\|^2+2\Delta \bm x^{k-1}\cdot({\bm v}_{i,0}^{k-1}-\bm x^{k-1})+\|\Delta \bm x^{k-1}\|^2)}_{T_{1,1}} \\
            & + {{p}}\frac{1}{n}\sum\nolimits_{i}\mathbb{E}\underbrace{\|(1-\lambda){\bm v}_{i,\tau}^{k-1}+\lambda\bm y_{i,\tau}^{k-1}-\bm x^k\|^2}_{T_{1,2}} \\
        \end{aligned}
    \end{equation}
    Bounding the two temporary terms with any $c > 0$
    \begin{equation}\nonumber
        \begin{aligned}
            T_{1,1} =& \mathbb{E}\|{\bm v}_{i,0}^{k-1}-\bm x^{k-1}\|^2+2\mathbb{E}[\Delta \bm x^{k-1}\cdot({\bm v}_{i,0}^{k-1}-\bm x^{k-1})]+\mathbb{E}\|\Delta \bm x^{k-1}\|^2 \\
            \leq & \mathbb{E}\|{\bm v}_{i,0}^{k-1}-\bm x^{k-1}\|^2+\mathbb{E}[\frac{1}{c}\|\mathbb{E}[\Delta \bm x^{k-1}]\|^2+c\|{\bm v}_{i,0}^{k-1}-\bm x^{k-1}\|^2]+\mathbb{E}\|\Delta \bm x^{k-1}\|^2 \\
            \leq & (1+c)\mathbb{E}\|{\bm v}_{i,0}^{k-1}-\bm x^{k-1}\|^2+\frac{1}{c}\mathbb{E}\|\mathbb{E}[\Delta \bm x^{k-1}]\|^2+\mathbb{E}\|\Delta \bm x^{k-1}\|^2
        \end{aligned}
    \end{equation}
    where we apply AM-GM and Cauchy-Schwarz inequalities in the first inequality.
    \begin{equation}\nonumber
        \begin{aligned}
            T_{1,2} \leq & 2(1-\lambda)^2\|{\bm v}_{i,0}^{k-1}-\bm x^{k-1}\|^2+8(1-\lambda)^2{\eta^2}{\tau}\sum\nolimits_j\|\nabla f_i({\bm v}_{i,j}^{k-1})-\nabla f_i(\bm x^{k-1})\|^2+8(1-\lambda)^2\eta^2\tau^2\|\nabla f_i(\bm x^{k-1})\|^2 \\
            & +8\lambda^2\|\bm y_{i,\tau}^{k-1}-\bm x^{k-1}\|^2 + 8\|\Delta \bm x^{k-1}\|^2 + (1-\lambda)^2\eta^2\tau\varsigma^2
        \end{aligned}
    \end{equation}
    Finally, we obtain the assertion.
    \begin{equation}\nonumber
        \begin{aligned}
            \tilde\varphi^k
            \leq & (1-{{p}})\frac{1}{n}\sum\nolimits_{i}\left((1+c)\mathbb{E}\|{\bm v}_{i,0}^{k-1}-\bm x^{k-1}\|^2+\frac{1}{c}\mathbb{E}\|\mathbb{E}[\Delta \bm x^{k-1}]\|^2+\mathbb{E}\|\Delta \bm x^{k-1}\|^2\right) \\
            & + {{p}}\frac{1}{n}\sum\nolimits_{i}\mathbb{E}\bigg[2(1-\lambda)^2\|{\bm v}_{i,0}^{k-1}-\bm x^{k-1}\|^2+8(1-\lambda)^2{\eta^2}{\tau}\sum\nolimits_j\|\nabla f_i({\bm v}_{i,j}^{k-1})-\nabla f_i(\bm x^{k-1})\|^2 \\
            & +8(1-\lambda)^2\eta^2\tau^2\|\nabla f_i(\bm x^{k-1})\|^2+8\lambda^2\|\bm y_{i,\tau}^{k-1}-\bm x^{k-1}\|^2 + 8\|\Delta \bm x^{k-1}\|^2 + (1-\lambda)^2\eta^2\tau\varsigma^2\bigg] \\
            \leq & ((1-{{p}})(1+c)+2{{p}}(1-\lambda)^2)\tilde{\varphi}^{k-1} + (1+7{{p}})\mathbb{E}\|\Delta \bm x^{k-1}\|^2+(1-{{p}})\frac{1}{c}\mathbb{E}\|\mathbb{E}[\Delta \bm x^{k-1}]\|^2\\
            & +8{{p}}(1-\lambda)^2\eta^2\tau^2\beta^2\varphi^{k-1}+ 8{{p}}(1-\lambda)^2\eta^2\tau^2(G^2+B^2\mathbb{E}\|\nabla f(\bm x^{k-1})\|^2)+8{{p}}\lambda^2\zeta^{k-1} + {{p}}(1-\lambda)^2\eta^2\tau\varsigma^2
        \end{aligned}
    \end{equation}
\end{proof}
 
\begin{lemma}[Estimating Sequence with Bounded Deviation]\label{proof.auxiliary}
    For any constant $\varepsilon \in \mathbb R^+$, if there holds $\lambda \geq \frac{1}{2}$, $\rho \leq \eta\beta$, and $\eta\tau\beta \leq \min \{\frac{1}{10}, \frac{1}{2\sqrt{3}}(\frac{p}{2q})^\frac{1}{4}\}$, we have the following
    \begin{equation}\nonumber
        \begin{aligned}
            \frac{1}{\varepsilon}\eta\tau\beta\varphi^k&+\frac{20}{\varepsilon}(\frac{1}{p}-\frac{1}{8})\eta\tau\beta\tilde\varphi^k  \leq\frac{1}{\varepsilon}\eta\tau\beta\varphi^{k-1}+\frac{20}{\varepsilon}(\frac{1}{p}-\frac{1}{8})\eta\tau\beta\tilde\varphi^{k-1} + (-\frac{5}{12}(1-\frac{1}{8}p)) \frac{1}{\varepsilon} \eta\tau\beta\tilde{\varphi}^{k-1} 
            \\ & + ( (24c_p q \eta^4\tau^4\beta^4 + \frac{12 c_p(1-{{p}})^2}{{{p}}} + 4c_p + 62c_p {{p}})\eta^2\tau^2\beta^2 - 1)\frac{1}{\varepsilon}\eta\tau\beta\varphi^{k-1}
            \\ & + (c_p Q B^2+ \frac{12c_p (1-{{p}})^2}{{{p}}} +4c_p (1+7{{p}})\tilde{B}^2) \frac{1}{\varepsilon}\eta^3\tau^3\beta  \mathbb{E}\|\nabla f(\bm x^{k-1})\|^2 + \frac{25}{12}\frac{1}{\varepsilon} \eta^3\tau^3\beta B^2\mathbb{E}\|\nabla f(\bm x^k)\|^2
            \\ & + (\frac{33c_p {{p}}}{4}+ \frac{(1+7{{p}})c_p}{m}+ \frac{25}{24})\frac{1}{\varepsilon}\eta^3\tau^2\beta\varsigma^2 + 2 c_p q\frac{1}{\varepsilon} \eta^5\tau^4\beta^3\varsigma^2 + 6 c_p q\frac{1}{\varepsilon}\eta^7\tau^6\beta^5\varsigma^2
            \\ & + (c_p Q + \frac{25}{12})\frac{1}{\varepsilon}\eta^3\tau^3\beta G^2 +  4c_p (1+7{{p}}) \frac{1}{\varepsilon}\eta^3\tau^3\beta \tilde{G}^2   
        \end{aligned}
    \end{equation}
    where $c_p := {20}(\frac{1}{p}-\frac{1}{48})$, $q := 5+75{{p}} + \frac{15(1-{{p}})^2}{{{p}}}$, and $Q:= 6q \eta^2\beta^2\tau^2 (1 + 4\eta^2\beta^2\tau^2)  + 34{{p}} \leq \frac{1}{3} + 39p + \frac{(1-{{p}})^2}{{{p}}}$.
\end{lemma}
\begin{proof}
    Let $c := \frac{p}{3(1-p)} $ mentioned in Lemma 2, for any $h_1, h_2 \geq 0$, there holds
    \begin{equation}\nonumber
        \begin{aligned}
            {h_1}\varphi^k+{h_2}\tilde\varphi^k
            \leq &\frac{{h_1}}{1-4\eta^2\tau^2\beta^2}(2\eta^2\tau^2(G^2+B^2\mathbb{E}\|\nabla f(\bm x^k)\|^2) + \eta^2\tau\varsigma^2)\\
            & + (\frac{2{h_1}}{1-4\eta^2\tau^2\beta^2}+{h_2})(((1-{{p}})(1+c)+2{{p}}(1-\lambda)^2)\tilde{\varphi}^{k-1} + (1+7{{p}})\mathbb{E}\|\Delta \bm x^{k-1}\|^2\\
            & +(1-{{p}})\frac{1}{c}\mathbb{E}\|\mathbb{E}[\Delta \bm x^{k-1}]\|^2+8{{p}}(1-\lambda)^2\eta^2\tau^2\beta^2\varphi^{k-1}+ 8{{p}}(1-\lambda)^2\eta^2\tau^2(G^2+B^2\mathbb{E}\|\nabla f(\bm x^{k-1})\|^2)
            \\ & +8{{p}}\lambda^2\zeta^{k-1} + {{p}}(1-\lambda)^2\eta^2\tau\varsigma^2)\\ 
            \leq &{h_1}\varphi^{k-1}+{h_2}\tilde\varphi^{k-1} + h_2^\prime Q\eta^2\tau^2 G^2 + 2h_1^\prime\eta^2\tau^2 G^2 + h_2^\prime 4(1+7{{p}}) \eta^2\tau^2 \tilde{G}^2
            \\ & + (h_2^\prime (1-\frac{1}{6}p + 12q \eta^4\beta^4\tau^4) - h_2) \tilde{\varphi}^{k-1} + (h_2^\prime (24q \eta^4\tau^4\beta^4  + \frac{12(1-{{p}})^2}{{{p}}} + 4 + 62{{p}})\eta^2\tau^2\beta^2 - h_1)\varphi^{k-1}
            \\ & + h_2^\prime (Q B^2+ \frac{12(1-{{p}})^2}{{{p}}} +4(1+7{{p}})\tilde{B}^2) \eta^2\tau^2 \mathbb{E}\|\nabla f(\bm x^{k-1})\|^2 + 2h_1^\prime\eta^2\tau^2B^2\mathbb{E}\|\nabla f(\bm x^k)\|^2
            \\ & + h_1^\prime\eta^2\tau\varsigma^2 + h_2^\prime (\frac{33}{4}{{p}}\eta^2\tau\varsigma^2 + 2q \eta^2\beta^2\tau^2 (1 + 3\eta^2\beta^2\tau^2)\eta^2\tau\varsigma^2 + (1+7{{p}}) \frac{\eta^2\tau\varsigma^2}{m})\\     
        \end{aligned}
    \end{equation}
    where $h_1^\prime:= \frac{h_1}{1-4\eta^2\tau^2\beta^2}$ and $h_2^\prime:= 2h_1^\prime+h_2$, then
    let $h_1:= \frac{1}{\varepsilon}\eta\tau\beta$, $h_2:= \frac{20}{\varepsilon}(\frac{1}{p}-\frac{1}{8})\eta\tau\beta$, $h_1^\prime := \frac{25}{24\varepsilon}\eta\tau\beta$, and $h_2^\prime := \frac{20}{\varepsilon}(\frac{1}{p}-\frac{1}{48})\eta\tau\beta$, we have
    \begin{equation}\nonumber
        \begin{aligned}
            \frac{1}{\varepsilon}\eta\tau\beta\varphi^k& +\frac{20}{\varepsilon}(\frac{1}{p}-\frac{1}{8})\eta\tau\beta\tilde\varphi^k      
            \leq\frac{1}{\varepsilon}\eta\tau\beta\varphi^{k-1}+\frac{20}{\varepsilon}(\frac{1}{p}-\frac{1}{8})\eta\tau\beta\tilde\varphi^{k-1} + (-\frac{5}{12}(1-\frac{1}{8}p)) \frac{1}{\varepsilon} \eta\tau\beta\tilde{\varphi}^{k-1} 
            \\ & + ( (24c_p q \eta^4\tau^4\beta^4 + \frac{12 c_p(1-{{p}})^2}{{{p}}} + 4c_p + 62c_p {{p}})\eta^2\tau^2\beta^2 - 1)\frac{1}{\varepsilon}\eta\tau\beta\varphi^{k-1}
            \\ & +  (c_p Q B^2+ \frac{12c_p (1-{{p}})^2}{{{p}}} +4c_p (1+7{{p}})\tilde{B}^2) \frac{1}{\varepsilon}\eta^3\tau^3\beta  \mathbb{E}\|\nabla f(\bm x^{k-1})\|^2 + \frac{25}{12}\frac{1}{\varepsilon} \eta^3\tau^3\beta B^2\mathbb{E}\|\nabla f(\bm x^k)\|^2
            \\ & + (\frac{33c_p {{p}}}{4}+ \frac{(1+7{{p}})c_p}{m}+ \frac{25}{24})\frac{1}{\varepsilon}\eta^3\tau^2\beta\varsigma^2 + 2 c_p q\frac{1}{\varepsilon} \eta^5\tau^4\beta^3\varsigma^2 + 6 c_p q\frac{1}{\varepsilon}\eta^7\tau^6\beta^5\varsigma^2
            \\ & + (c_p Q + \frac{25}{12})\frac{1}{\varepsilon}\eta^3\tau^3\beta G^2 +  4c_p (1+7{{p}}) \frac{1}{\varepsilon}\eta^3\tau^3\beta \tilde{G}^2   
        \end{aligned}
    \end{equation}
\end{proof}
 
Having established preceding lemmas, we can derive the convergence result of the recursive form.
\begin{lemma}[One round progress]\label{proof.orp}
    Suppose that $\rho \leq \eta\beta$, $\eta\tau\beta \leq \min\{\frac{1}{144\tilde{B}^2}, \frac{1}{84\sqrt{2}\sqrt{l_p^1+l_p^2B^2+l_p^3\tilde{B}^2}}\}$, we have
    \begin{equation}\nonumber
        \begin{aligned}
            \mathbb{E}&[f(\bm{x}^k+\Delta \bm{x}^k)] + 2\eta\tau\beta^2\varphi^{k+1}+ {40}(\frac{1}{p}-\frac{1}{8})\eta\tau\beta^2\tilde\varphi^{k+1}
            \\ \leq & f(\bm{x}^k) + 2\eta\tau\beta^2\varphi^{k}+ {40}(\frac{1}{p}-\frac{1}{8})\eta\tau\beta^2\tilde\varphi^{k} - \frac{1}{24}\eta\tau \mathbb{E}\|\nabla f(\bm x^{k})\|^2
            \\ & + \frac{25}{6} \eta^3\tau^3\beta^2 B^2\bigg(\mathbb{E}\|\nabla f(\bm x^{k+1})\|^2 - \mathbb{E}\|\nabla f(\bm x^k)\|^2\bigg) + \eta^2\tau^2\beta \bigg(2 \tilde{G}^2 + \frac{1}{2}\frac{\varsigma^2}{\tau m}\bigg) 
            \\ & + \eta^3\tau^3\beta^2 \bigg( (1120 + \frac{160}{p}) \tilde{G}^2 +  (1548 + \frac{25}{2p} + \frac{75}{2}\frac{(1-p)^2}{p^2} + \frac{97}{6}) G^2 + ({330 {{p}}} + \frac{{40}}{m{p}} + \frac{280}{m} + \frac{73}{12}) \frac{\varsigma^2 }{\tau}\bigg) 
            \\ & +\eta^4\tau^4\beta^3 \bigg(24 G^2 +  \frac{8\varsigma^2}{\tau}\bigg) + \eta^5\tau^5\beta^4 \bigg(48 G^2 + \frac{(12p + {80q} )\varsigma^2}{p\tau}\bigg) + \eta^6\tau^6\beta^5 \bigg(96 G^2 + \frac{24\varsigma^2}{\tau}\bigg)  + \eta^7\tau^7\beta^6 \frac{240 q}{p} \frac{\varsigma^2}{\tau}
        \end{aligned}
    \end{equation}
    where $l_p^1:= \frac{15 (1-{{p}})^2}{49{{p}}^2}$, $l_p^2:= 1 + \frac{25}{3136p} + \frac{75(1-p)^2}{3136p^2}$, and $l_p^3:= \frac{5}{7}+\frac{5}{49p}$.
\end{lemma}
\begin{proof}
    We begin with the property of smoothness directly. For any $\varepsilon > 0$, there holds
    \begin{equation}\label{{proof.orp.f1}}
        \begin{aligned}
            \mathbb{E}&[f(\bm{x}^k+\Delta \bm{x}^k)]-f(\bm{x}^k) \leq \nabla f(\bm{x}^k)\cdot\mathbb{E}[\Delta \bm{x}^k]+\frac{\beta}{2} \mathbb{E}\|\Delta \bm{x}^k\|^{2}
            \\ \leq & - \nabla f(\bm{x}^k)\cdot \mathbb{E}[\frac{1}{m}\sum\nolimits_{i \in \mathcal{U}^k}(\eta\sum\nolimits_{j=0}^{\tau}{g}_{i,j}^k + \rho \sum\nolimits_{j=0}^{\tau}({\bm v}_{i,j}^k + \bm y_{i,j}^k - 2\bm x^k))]+\frac{\beta}{2} \mathbb{E}\|\Delta \bm{x}^k\|^{2}\\
            \leq & - \frac{1}{2}\eta\tau\left(\mathbb{E}\|\nabla f(\bm{x}^k)\|^2 - \beta^2\psi^k\right) + \frac{1}{2}\rho\tau\left(\frac{\varepsilon}{2}\mathbb{E}\|\nabla f(\bm{x}^k)\|^2 + \frac{2}{\varepsilon}\varphi^k\right)
            \\ & + \frac{1}{2}\rho\tau\left(\frac{\varepsilon}{2}\mathbb{E}\|\nabla f(\bm{x}^k)\|^2 + \frac{2}{\varepsilon}\psi^k\right) + \frac{\beta}{2} \mathbb{E}\|\Delta \bm{x}^k\|^{2}\\
            \leq & - (\frac{1}{2}\eta\tau-\frac{\varepsilon}{2}\rho\tau)\mathbb{E}\|\nabla f(\bm{x}^k)\|^2 + (\frac{1}{2}\eta\tau\beta^2+\frac{1}{\varepsilon}\rho\tau+2\beta(\eta^2\beta^2+\rho^2)\tau^2)\psi^k
            \\ & + (\frac{1}{\varepsilon}\rho\tau + 2\beta\rho^2\tau^2) \varphi^k + 2\beta\eta^2\tau^2(\tilde{G}^2+\tilde{B}^2\mathbb{E}\|\nabla f(\bm x^k)\|^2) + \frac{\beta}{2}\frac{\eta^2\tau\varsigma^2}{m}
            \\ {\leq}& (- \frac{1}{6}\eta\tau +2\beta\eta^2\tau^2 \tilde{B}^2 + 12(1+2\eta\beta\tau) (1 + 4\eta^2\beta^2\tau^2)\eta^3\tau^3\beta^2 B^2)\mathbb{E}\|\nabla f(\bm x^k)\|^2
            \\ & + 24(1+2\eta\beta\tau)\eta^3\tau^3\beta^4\tilde{\varphi}^k + (\frac{3}{2}\eta\tau\beta^2 + 2\beta\eta^2\beta^2\tau^2) \varphi^k + 48(1+2\eta\beta\tau)\eta^5\tau^5\beta^6\varphi^k
            \\ & + 2\beta\eta^2\tau^2 \tilde{G}^2 + 12(1+2\eta\beta\tau) (1 + 4\eta^2\beta^2\tau^2)\eta^3\tau^3\beta^2 G^2 + \frac{\beta}{2}\frac{\eta^2\tau\varsigma^2}{m} + 4(1+2\eta\beta\tau)(1 + 3\eta^2\beta^2\tau^2)\eta^3\tau^2\beta^2\varsigma^2
        \end{aligned}
    \end{equation}
    where in the last inequality, we use $\varepsilon := \frac{2}{3\beta}$ and Lemmas 2-4. Then, we add $\frac{4}{3}$ of Lemma 5 to each side of (\ref{{proof.orp.f1}}) to derive
    \begin{equation}\label{proof.orp.f2}
        \begin{aligned}
            \mathbb{E}&[f(\bm{x}^k+\Delta \bm{x}^k)] + 2\eta\tau\beta^2\varphi^{k+1}+ {40}(\frac{1}{p}-\frac{1}{8})\eta\tau\beta^2\tilde\varphi^{k+1}
            \\ \leq & f(\bm{x}^k) + 2\eta\tau\beta^2\varphi^{k}+ {40}(\frac{1}{p}-\frac{1}{8})\eta\tau\beta^2\tilde\varphi^{k} - (\frac{1}{24}\eta\tau - P_1) \mathbb{E}\|\nabla f(\bm x^{k})\|^2 + P_2\tilde{\varphi}^{k} + P_3\varphi^{k}
            \\ & + \frac{25}{6} \eta^3\tau^3\beta^2 B^2\bigg(\mathbb{E}\|\nabla f(\bm x^{k+1})\|^2 - \mathbb{E}\|\nabla f(\bm x^k)\|^2\bigg) + \eta^2\tau^2\beta \bigg(2 \tilde{G}^2 + \frac{1}{2}\frac{\varsigma^2}{\tau m}\bigg) 
            \\ & + \eta^3\tau^3\beta^2 \bigg( (1120 + \frac{160}{p}) \tilde{G}^2 +  (1548 + \frac{25}{2p} + \frac{75}{2}\frac{(1-p)^2}{p^2} + \frac{97}{6}) G^2 + ({330 {{p}}} + \frac{{40}}{m{p}} + \frac{280}{m} + \frac{73}{12}) \frac{\varsigma^2 }{\tau}\bigg) 
            \\ & +\eta^4\tau^4\beta^3 \bigg(24 G^2 +  \frac{8\varsigma^2}{\tau}\bigg) + \eta^5\tau^5\beta^4 \bigg(48 G^2 + \frac{(12p + {80q} )\varsigma^2}{p\tau}\bigg) + \eta^6\tau^6\beta^5 \bigg(96 G^2 + \frac{24\varsigma^2}{\tau}\bigg)  + \eta^7\tau^7\beta^6 \frac{240 q}{p} \frac{\varsigma^2}{\tau}.
        \end{aligned}
    \end{equation}
    We use $\eta\tau\beta \leq \min\{\frac{1}{144\tilde{B}^2}, \frac{1}{84\sqrt{2}\sqrt{l_p^1+l_p^2B^2+l_p^3\tilde{B}^2}}\}$ to guarantee $P_i \leq 0, i=1,2,3$, where polynomials w.r.t. $\eta$ have the following
    \begin{equation}\nonumber
        \begin{aligned}
            P_1 := & - \frac{1}{8}\eta\tau +2\beta\eta^2\tau^2 \tilde{B}^2 + (20 B^2 + 2 c_p Q B^2+ \frac{24c_p (1-{{p}})^2}{{{p}}} +8c_p (1+7{{p}})\tilde{B}^2)\eta^3\tau^3\beta^2
            \\ & - (1+\frac{4n}{m}){256}\eta^3\tau^3\beta^2(2(\frac{n}{m}+{{p}}-2) + (\frac{2m}{n}\frac{(1-\lambda)^2}{\tau^2}+\frac{8m}{n}\lambda^2)B^2+(1+\frac{7m}{n})\tilde{B}^2+150l\eta^2\tau^2\beta^2 B^2)
            \\ P_2 := & 24(1+2\eta\beta\tau)\eta^3\tau^3\beta^4 + (-\frac{5}{6}(1-\frac{1}{8}p)) \eta\tau\beta^2
            \\ P_3 := & (\frac{3}{2}\eta\tau\beta^2 + 2\beta\eta^2\beta^2\tau^2) + 48(1+2\eta\beta\tau)\eta^5\tau^5\beta^6 + ( (24c_p q \eta^4\tau^4\beta^4 + \frac{12 c_p(1-{{p}})^2}{{{p}}} + 4c_p + 62c_p {{p}})\eta^2\tau^2\beta^2 - 1)2\eta\tau\beta^2
        \end{aligned}
    \end{equation}
    Finally, we obtain the result by simplifying the formula (\ref{proof.orp.f2}).
\end{proof}

\section{Proof of Theorems}
Using Lemma \ref{proof.orp}, we easily obtain the convergence bound of FedDeper in the sense of Ces\`aro means.
 
\begin{thmbis}{thm.deper}
    Suppose that each loss function $ (f_i) $ meets Assumptions \ref{A1}, \ref{A2}, and \ref{A3}(i). Then the proposed FL method satisfies:
    \begin{equation}\nonumber
        \begin{aligned}
            & \frac{1}{K}\sum\nolimits_{k=0}^{K-1}\mathbb{E}\|\nabla f(\bm x^{k})\|^2 \leq \frac{24\varGamma}{\eta\tau K} + 12 \eta\tau\beta \bigg(4 \tilde{G}^2 + \frac{\varsigma^2}{\tau m}\bigg) 
            \\ & + 24 \eta^2\tau^2\beta^2 \bigg( (1120 + \frac{160}{p}) \tilde{G}^2 +  (1548 + \frac{25}{2p} + \frac{75}{2}\frac{(1-p)^2}{p^2} + \frac{97}{6}) G^2 + ({330 {{p}}} + \frac{{40}}{m{p}} + \frac{280}{m} + \frac{73}{12}) \frac{\varsigma^2 }{\tau}\bigg) 
            \\ & + 192 \eta^3\tau^3\beta^3 \bigg(3 G^2 +  \frac{\varsigma^2}{\tau}\bigg) + 96 \eta^4\tau^4\beta^4 \bigg(12 G^2 + \frac{(3p + {20q} )\varsigma^2}{p\tau}\bigg) + 576 \eta^5\tau^5\beta^5 \bigg(4 G^2 + \frac{\varsigma^2}{\tau}\bigg)  + 5760 \eta^6\tau^6\beta^6 \frac{q\varsigma^2}{p\tau}
        \end{aligned}
    \end{equation}
    where $ \varGamma := f(\bm x^0) - f(\bm x^{*}) $, $\tilde{B}^2 := 2B^2(\frac{1}{m}-\frac{1}{n})+1$ and $\tilde{G}^2 := 2G^2(\frac{1}{m}-\frac{1}{n})$.
\end{thmbis}
\begin{proof}
    For the sake of convenience, we let $ {F}^{k}:=f(\bm{x}^k) + 2\eta\tau\beta^2\varphi^{k}+ {40}(\frac{1}{p}-\frac{1}{8})\eta\tau\beta^2\tilde\varphi^{k}$. Then we rewrite Lemma \ref{proof.orp} as follows
    \begin{equation}\nonumber
        \begin{aligned}
            \frac{1}{24}& \eta\tau \mathbb{E}\|\nabla f(\bm x^{k})\|^2 \leq {F}^{k} - {F}^{k+1} + \frac{25}{6} \eta^3\tau^3\beta^2 B^2\bigg(\mathbb{E}\|\nabla f(\bm x^{k+1})\|^2 - \mathbb{E}\|\nabla f(\bm x^k)\|^2\bigg) + \eta^2\tau^2\beta \bigg(2 \tilde{G}^2 + \frac{1}{2}\frac{\varsigma^2}{\tau m}\bigg) 
            \\ & + \eta^3\tau^3\beta^2 \bigg( (1120 + \frac{160}{p}) \tilde{G}^2 +  (1548 + \frac{25}{2p} + \frac{75}{2}\frac{(1-p)^2}{p^2} + \frac{97}{6}) G^2 + ({330 {{p}}} + \frac{{40}}{m{p}} + \frac{280}{m} + \frac{73}{12}) \frac{\varsigma^2 }{\tau}\bigg) 
            \\ & +\eta^4\tau^4\beta^3 \bigg(24 G^2 +  \frac{8\varsigma^2}{\tau}\bigg) + \eta^5\tau^5\beta^4 \bigg(48 G^2 + \frac{(12p + {80q} )\varsigma^2}{p\tau}\bigg) + \eta^6\tau^6\beta^5 \bigg(96 G^2 + \frac{24\varsigma^2}{\tau}\bigg)  + \eta^7\tau^7\beta^6 \frac{240 q}{p} \frac{\varsigma^2}{\tau}
        \end{aligned}
    \end{equation}
    Directly we establish the convergence rate by averaging over $k$ on both sides of the preceding.
\end{proof}

\begin{thmbis}{thm.per}
    Let $\frac{1}{n\tau K}\sum\nolimits_{i,j,k}(\cdot)$ average over all the indexes $i,j,k$, (i) in terms of Theorem \ref{thm.deper}, for any $\lambda \in [\frac{1}{2}, 1]$, we have 
    \begin{equation}\nonumber
        \begin{aligned}
            \frac{1}{n\tau K}\sum\nolimits_{i,j,k}\|\bm v_{i,j}^k - \bm x^*\|^2 \leq \mathcal{O}({\xi^{0}}) + \mathcal{O}(\epsilon),
        \end{aligned}
    \end{equation}
    and (ii) in terms of Corollary \ref{thm.rate}, for $\lambda = 1$, we have
    \begin{equation}\nonumber
        \begin{aligned}
            \frac{1}{n\tau K}\sum\nolimits_{i,j,k}\|\bm v_{i,j}^k - \bm x^*\|^2 \leq \mathcal{O}(\epsilon),
        \end{aligned}
    \end{equation}
    where $\mathcal{O}$ hides all constants, $\xi^0 := \frac{1}{n\tau}\sum\nolimits_{i,j}\mathbb{E}\|{\bm v}_{i,j}^0-\bm x^0\|^2 $, and $\epsilon := \frac{1}{K}\sum\nolimits_{k=0}^{K-1}\mathbb{E}\|\nabla f(\bm x^{k})\|^2$.
\end{thmbis}
\begin{proof}
    We directly start by bounding the following term with the triangle inequality
    \begin{equation}\nonumber
        \begin{aligned}
            \frac{1}{n\tau K}\sum\nolimits_{i,j,k}\|\bm v_{i,j}^k - \bm x^*\|^2 &\leq \frac{1}{n\tau K}\sum\nolimits_{i,j,k}\bigg(\|\bm v_{i,j}^k - \bm x^*\|^2 + {20}(\frac{1}{p}-\frac{1}{8}) \|\bm v_{i,0}^k - \bm x^*\|^2\bigg)
            \\ & = \frac{1}{K}\sum\nolimits_k\frac{1}{n\tau}\sum\nolimits_{i,j}\|\bm v_{i,j}^k - \bm x^*\|^2 + {20}(\frac{1}{p}-\frac{1}{8})\frac{1}{K}\sum\nolimits_k\frac{1}{n}\sum\nolimits_{i} \|\bm v_{i,0}^k - \bm x^*\|^2
            \\ & = 2\frac{1}{K}\sum\nolimits_k\underbrace{(\varphi^k + {20}(\frac{1}{p}-\frac{1}{8})\tilde{\varphi}^k)}_{= \xi^k} + 4\frac{1}{K}\sum\nolimits_k\|\bm x^k - \bm x^*\|^2 \leq 2\frac{1}{K}\sum\nolimits_k \xi^k + \frac{4}{\beta^2}\epsilon,
        \end{aligned}
    \end{equation}
    where the last inequality is due to the smoothness of $f(\cdot)$ with parameter $\beta$.
    Then by using Lemma \ref{proof.auxiliary} with $\varepsilon:= \eta\tau\beta$, we have
    \begin{equation}\nonumber
        \begin{aligned}
            \frac{1}{K}\sum\nolimits_{k=0}^{K-1}\xi^{k}  &\leq \xi^{0} -\xi^K + (1 - \frac{1}{48}p)\frac{1}{K}\sum\nolimits_{k=0}^{K-1}\xi^{k} + 2 c_p q \eta^4\tau^3\beta^2\varsigma^2 + 6 c_p q\eta^6\tau^5\beta^4\varsigma^2
            \\ & + (\frac{12c_p (1-{{p}})^2}{{{p}}} + 4c_p (1+7{{p}})\tilde{B}^2 +c_p Q B^2+ \frac{25}{12} B^2)\eta^2\tau^2 \frac{1}{K}\sum\nolimits_{k=0}^{K-1}\mathbb{E}\|\nabla f(\bm x^{k})\|^2 
            \\ & + (\frac{33c_p {{p}}}{4}+ \frac{(1+7{{p}})c_p}{m}+ \frac{25}{24})\eta^2\tau\varsigma^2 + \eta^2\tau^2((c_p Q + \frac{25}{12})G^2 +  4c_p (1+7{{p}})\tilde{G}^2),
        \end{aligned}
    \end{equation}
    where $\xi^0 = \varphi^0 + {20}(\frac{1}{p}-\frac{1}{8})\tilde{\varphi}^0 = \frac{1}{n\tau}\sum\nolimits_{i,j}\mathbb{E}\|{\bm v}_{i,j}^0-\bm x^0\|^2 + {20}(\frac{1}{p}-\frac{1}{8})\frac{1}{n}\sum\nolimits_{i}\mathbb{E}\|{\bm v}_{i,0}^0-\bm x^0\|^2= \frac{1}{n\tau}\sum\nolimits_{i,j}\mathbb{E}\|{\bm v}_{i,j}^0-\bm x^0\|^2 $.
    Finally we scale $\frac{1}{K}\sum\nolimits_k \xi^k$ as
    \begin{equation}\nonumber
        \begin{aligned}
            \frac{1}{K}\sum\nolimits_{k=0}^{K-1}\xi^{k} \leq& \frac{1}{p}\mathcal{O}\bigg(\xi^{0} + \frac{m\tau\epsilon}{K} (\frac{(1-{{p}})^2}{{{p}^2}} + (1+\frac{1}{p})\tilde{B}^2 + (1+\frac{1}{p}) B^2)+ \frac{m\tau}{K}(1+ \frac{1}{mp} + \frac{1}{m})\frac{\varsigma^2}{\tau} + \frac{m\tau}{K}((1+\frac{1}{p})G^2 
            \\ & +  (1+\frac{1}{p})\tilde{G}^2)+ (\frac{m\tau}{K})^2(1+\frac{1}{p}+\frac{(1-p)^2}{p^2}) \frac{\beta^2\varsigma^2}{\tau} + (\frac{m\tau}{K})^3(1+\frac{1}{p}+\frac{(1-p)^2}{p^2}) \frac{\beta^4\varsigma^2}{\tau}\bigg)
            \\ \leq& \mathcal{O}({\xi^{0}}) + \frac{1}{p}\mathcal{O}\bigg(\frac{m\tau\epsilon}{K} (1 + B^2)+ \frac{m\tau}{K}(1 + \frac{1}{m})\frac{\varsigma^2}{\tau} + \frac{m\tau}{K}G^2 + (\frac{m\tau}{K})^2\frac{\beta^2\varsigma^2}{\tau} + (\frac{m\tau}{K})^3 \frac{\beta^4\varsigma^2}{\tau}\bigg)
        \end{aligned}
    \end{equation}
    We complete part (i) by using $\mathcal{O}(\epsilon)$ to swallow the second part in the RHS of the last inequality.
    While part (ii) is trivial with (i), which is omitted in the proof.
\end{proof}

\end{document}